%% file: arxiv.tex
\documentclass[letterpaper,11pt]{article}
\usepackage{amsmath, amsthm, amssymb}
\usepackage[usenames, dvipsnames]{color}
\usepackage[normalem]{ulem} 
\usepackage{fullpage}
\usepackage[numbers, sort]{natbib}
\usepackage[noend]{algorithmic}
\usepackage{tikz, pgfplots}
\usepackage{enumerate}
\usepackage{relsize}
\usepackage{hyperref}
\usepackage{xspace,color}
\usepackage{graphicx}
\usepackage{caption}
\usepackage{subcaption}
\usepackage{enumitem,linegoal}
\usepackage[linesnumbered,ruled,vlined]{algorithm2e}
\usepackage{comment}	
\usepackage{ctable}			
\usepackage{xcolor,colortbl}

\newtheorem{observation}{Observation}[section]

\usepackage{mathtools}
\usepackage[flushleft]{threeparttable}
\usepackage{verbatim}
\usepackage{ifthen}
\usepackage{amsmath}
\usepackage{bbm}
\usepackage{footnote}
\makesavenoteenv{tabular}
\makesavenoteenv{table}

\usepackage[margin=1in]{geometry}

\definecolor{crimsonglory}{rgb}{0,0,0}

 \newtheorem{theorem}{Theorem}[section]
 
 \newtheorem{lemma}[theorem]{Lemma}
 
 \newtheorem{corollary}[theorem]{Corollary}

\makeatletter
\def\GrabProofArgument[#1]{ #1: \egroup\ignorespaces}
\def\proof{\noindent\textbf\bgroup Proof%
	\@ifnextchar[{\GrabProofArgument}{. \egroup\ignorespaces}}

\makeatother

\usetikzlibrary{arrows,shapes,snakes,automata,backgrounds,petri,calc}
\usepackage[latin1]{inputenc}

\input{src/macros.tex}
\usepackage{float}

\newcounter{proccnt}

\newcommand{\konote}[1]{}

\usepackage[normalem]{ulem} 

\title{Learning Complexity of Simulated Annealing}

\author{
	Avrim Blum\thanks{Toyota Technological Institute at Chicago, supported in part by the National Science Foundation under grants CCF-1733556, CCF-1800317, and CCF-1815011.}
	\and Chen Dan\thanks{CMU, supported in part by the National Science Foundation under grant CCF-1800317.}
	\and Saeed Seddighin\thanks{Toyota Technological Institute at Chicago, supported in part by the National Science Foundation under grants CCF-1733556 and CCF-1535795.}
}

\begin{document}
	\newcommand{\ignore}[1]{}
\renewcommand{\theenumi}{(\roman{enumi}).}
\renewcommand{\labelenumi}{\theenumi}
\sloppy
\date{}

\maketitle

\thispagestyle{empty}

\begin{abstract}
	\input{src/abstract.tex}
	\newpage
\end{abstract}
\input{src/introduction}

\input{src/discretization}

\input{src/lowerbound}

\input{src/improved-upperbound}
\input{src/computational-model}
\input{src/learning}
\input{src/bestsequence}

\newpage
\bibliographystyle{abbrv}	
\bibliography{src/learning}
\newpage
\appendix
\input{src/lowerbound-appendix}
\input{src/improved-upperbound-appendix}
\end{document}

%% file: src/macros.tex
\renewcommand{\tilde}{\widetilde}
\renewcommand{\function}{f}
\newcommand{\gad}{\mathcal{G}}
\newcommand{\key}{\mathcal{K}}
\newcommand{\distribution}{\mathcal{D}}
\newcommand{\dtwo}{\mathcal{R}}
\newcommand{\stationary}{\mathcal{S}}
\newcommand{\empirical}{\mathcal{S}}
\newcommand{\feasible}{\mathcal{F}}
\newcommand{\opt}{\mathsf{OPT}}
\newcommand{\cross}{\mathsf{crossingl}}
\newcommand{\length}{\ell}
\newcommand{\allsatisfied}{\textsf{all-satisfied}\ }
\newcommand{\poly}{\mathsf{poly}}
\newcommand{\para}{\mathcal{E}}
\newcommand{\ii}{\mathsf{I}}

\newcommand{\score}{\mathsf{score}}
\newcommand{\graph}{\mathcal{G}}
\newcommand{\tsize}{|T|}
\newcommand{\ione}{\mathsf{I_1}}
\newcommand{\itwo}{\mathsf{I_2}}
\newcommand{\emax}{\mathsf{e}_{\text{max}}}
\newcommand{\optd}{\mathsf{OPT(D)}}
\newcommand{\opts}{\mathsf{OPT(S)}}
\newcommand{\E}{\mathbb{E}}

%% file: src/abstract.tex
\textit{Simulated annealing} is an effective and general means of optimization. It is in fact inspired by metallurgy, where the temperature of a material determines its behavior in thermodynamics. Likewise, in simulated annealing, the actions that the algorithm takes depend entirely on the value of a variable which captures the notion of temperature. Typically, simulated annealing starts with a high temperature, which makes the algorithm pretty unpredictable, and gradually cools the temperature down to become more stable.

A key component that plays a crucial role in the performance of simulated annealing is the criteria under which the temperature changes namely,  \textit{the cooling schedule}.  Motivated by this, we study the following question in this work: ``\textit{Given enough samples to the instances of a specific class of optimization problems, can we design optimal (or approximately optimal) cooling schedules that minimize the runtime or maximize the success rate of the algorithm on average when the underlying problem is drawn uniformly at random from the same class?}" 

We provide positive results both in terms of \textit{sample complexity} and \textit{simulation complexity}\footnote{We call the overall runtime of the algorithm that determines the cooling schedule the simulation complexity}. For sample complexity, we show that $\tilde O(\sqrt{m})$ samples suffice to find an approximately optimal cooling schedule of length $m$. We complement this result by giving a lower bound of $\tilde \Omega(m^{1/3})$ on the sample complexity of any learning algorithm that provides an almost optimal cooling schedule. These results are general and rely on no assumption. For simulation complexity, however, we make additional assumptions to measure the success rate of an algorithm. To this end, we introduce \textit{the monotone stationary graph} that models the performance of simulated annealing. Based on this model, we present polynomial time algorithms with provable guarantees for the learning problem. 

%% file: src/introduction.tex
\section{Introduction}\label{section:introduction}
The goal of this work is to better understand how we can design efficient simulated annealing (\textsf{SA}) algorithms. Simulated annealing is a well-known heuristic method to tackle hard problems.  Term annealing originates from thermodynamics, referring to the way that metals cool and anneal.  Instead of the energy of the material, simulated annealing utilizes the objective function of an optimization problem. Surprisingly, the implementation of \textsf{SA} is very simple as it is very similar to hill-climbing.  The only difference is that instead of picking the best move in every step, simulated annealing picks a random move.  If the selected move improves the quality of the solution, then the move is always accepted.  Otherwise, the algorithm makes the move anyway with some probability less than 1.  The probability decreases exponentially with the \textit{badness} of the move, which is the amount by which the solution is worsened. This is shown by $\Delta(E)$. One example of the annealing criteria is given below:

$$\mathsf{Pr}[\text{ accepting a downhill move at time step } i] \simeq 1 - e^{\Delta(E)/ t_i}.$$

where parameter $t_i$ is the temperature of the algorithm at step $i$ which is used to determine this probability.  The $t_i$ parameter is analogous to temperature in an annealing system at time step $i$.  At higher values of temperature, downhill moves are more likely to occur.  As the temperature tends to zero, they become more and more unlikely, until the algorithm behaves more or less like hill-climbing.  In a typical \textsf{SA} optimization, the temperature starts at a high value and is gradually decreased according to a cooling schedule.  Simulated annealing is used for a broad class of computational problems ranging from  \textsf{SAT} to travelling salesman problem, to VLSI routing, etc. as experiments strongly support the efficiency of simulated annealing in practice~\cite{aragon1984simulated,kirkpatrick1983optimization}.

Indeed the efficiency of an \textsf{SA} algorithm significantly depends on its cooling schedule~\cite{lam1988performance,lam1988efficient,nourani1998comparison,kirkpatrick1983optimization,saccostochastic,nulton1988statistical,triki2005theoretical,azizi2004adaptive,aarts1988simulated}. One simple cooling schedule is to start with a single temperature $t_0$ and decrease the temperature linearly with a rate of $\alpha$ to obtain lower temperatures gradually. We use this simple cooling strategy to present illustrating examples, nonetheless we consider a more generalized setting in this work.  The literature has also gone beyond simple cooling schedules and several non-linear methods have been proposed so far ~\cite{lam1988performance,lam1988efficient,nourani1998comparison,kirkpatrick1983optimization,saccostochastic,nulton1988statistical,triki2005theoretical,azizi2004adaptive,aarts1988simulated}. It is not hard to imagine that even for different instances of the same problem, the optimal cooling schedules may vary significantly. 

Therefore in this work, we take a learning approach towards designing simulated annealing algorithms, using the PAC-style model for data-driven algorithm design introduced in \cite{gupta2017pac} and used to analyze a wide range of important families of algorithms and heuristics in \cite{balcan2017learning,balcan2018learning,balcan2019much,BDL20}.  In brief, we consider a distribution $\mathcal{D}$ over a specific class of instances of a presumably hard problem (such as \textsf{SAT}) and aim to design near-optimal cooling schedules for such instances, analyzing both sample complexity (the number of instances from $\mathcal{D}$  we need to observe) and simulation complexity (runtime) needed for learning.   Our approach is particularly motivated by the work of \cite{balcan2017learning}.

\subsection{The Learning Problem}
As aforementioned, an \textsf{SA} algorithm makes a random walk on the nodes of a search graph. Each node of this graph represents a potential (not necessarily optimal) solution for the underlying problem and the energy of a node is a value reflecting how close this solution is to an optimal solution. We assume that for each node, its energy and neighbors are available via oracle queries. One thing to keep in mind is that the number of nodes in this huge search graph may be exponentially large and that we only have local views on the nodes of the graph. For instance, when the underlying problem is \textsf{SAT}, we may have $2^k$ nodes where $k$ is the number of variables in the \textsf{SAT} problem and each node represent an assignment of true/false to the variables.

Crucial to any cooling schedule are the parameters that maximize its performance. This could be as simple as just a real value specifying the cooling rate or as complicated as a sequence of variables determining the exact value of the temperature at every step. Take for instance, the simplest case in which a parameter $t_0$ and linear cooling rate $\alpha$ formulate the temperature at every step. In this case, at step $i$, $t_i = t_0 (1- \alpha i)$ formulates the temperature. Therefore, the learning algorithm has to find the optimal pair $(t_0,\alpha)$ that maximizes efficiency. It is an easy exercise to see that the learning problem is actually not very challenging in this case. Although this simple formulation involves infinitely many $(t_0,\alpha)$ pairs that need to be searched over, via careful discretization techniques, one can narrow down the set of possible $(t_0,\alpha)$ pairs to polynomially many candidates and iterate over them to find the optimal cooling schedule\footnote{This improvement comes with a small error to the quality of the solution.}. Samples are used to determine how well each cooling schedule performs in practice. More precisely, samples are used to approximate the score of a cooling schedule.

 However, we go beyond linear cooling schedules and include more sophisticated systems (\textit{i.e.,} non-linear cooling schedules). Our setting is pretty general: we denote the cooling schedule by a vector $\para = \langle t_1,t_2, \ldots, t_m\rangle$ where $m$ is the number of  steps our algorithm takes and $t_i$ specifies the temperature at time $i$. Any non-increasing sequence of values makes a valid cooling schedule. The problem becomes more challenging with this representation; Even after discretizing the temperatures, still there are exponentially many cooling schedules and determining an approximately optimal schedule is non-trivial.
 
 Recall that each node of the search graph corresponds to a potential solution for the underlying problem. In case of \textsf{SAT} for instance, each node can be an assignment of the true/false values to the variables. We label a subset of nodes in the search graph as \textit{acceptable solution nodes}. These nodes correspond to solutions that are acceptable for the underlying problem. In the case of \textsf{SAT}, a node whose corresponding solution satisfies all of the clauses is a solution node. The score of an \textsf{SA} algorithm with a specific cooling schedule is the likelihood of reaching an acceptable solution node after a fixed number of steps. We would like to point out that although a reasonable energy function for the nodes of the search graph gives higher energies to the acceptable solution nodes, we make no particular assumption on the energies in our setting. We remark that if the cooling schedule is available, it is computationally easy to evaluate the score of the algorithm. We run the \textsf{SA} algorithm according to the cooling schedule and once it terminates we find out if the solution found by the algorithm is acceptable for the problem. By repeating this process enough times, we can estimate the score of the cooling schedule very accurately.
 
  Indeed the optimal parameters may vary for different problems or even for different instances of the same problem and therefore we need to also incorporate the problem instances in our setting. To illustrate the importance of the cooling schedule, consider the example shown in Figure \ref{fig:superficial}.
  In the example shown in Figure \ref{fig:superficial}, there is one solution node (colored in red) which has an energy of $3n$ and the search graph consists of a clique of vertices with distinct energies one of whose vertices has a path to the solution node. The energies of the nodes of this path are increasing. Let us assume that the initial state of the algorithm is the node colored in green. It is easy to verify that an extreme strategy that never accepts downhill moves has zero chance of reaching the solution node and another extreme strategy that always accepts all downhill moves requires a cubic number of steps to reach the solution node. However, a strategy that accepts each downhill move with probability $1/2$ only requires $\mathcal{O}(n^2)$ steps in expectation to reach the solution node. 

  Motivated by this example, we define our general problem in the following way:

\input{figs/superficial}
Let $\distribution$ be a distribution over a specific class of instances\footnote{For instance industrial instances of \textsf{SAT}.} of a hard problem (such as \textsf{SAT}). Denote the set of valid (combination of) parameters for the \textsf{SA} algorithm by $\feasible = \{\para_1, \para_2, \ldots, \}$. Moreover, let for an instance $\mathsf{I} \sim \distribution$ and a set of parameters $\para \in \feasible$, $\score(\mathsf{I},\para)$ be a function that reflects how well an \textsf{SA} algorithm with parameters $\para$ works on instance $\mathsf{I}$. This is basically the likelihood of finding a solution, if our \textsf{SA} algorithm uses $\para$ as its cooling schedule. Our goal is to find a set of parameters $\para \in \feasible$ that maximizes efficiency. In other words,
$$\mathbb{E}_{\mathsf{I} \sim \distribution} [\score(\mathsf{I},\para)]$$ is (approximately) maximized.

We clarify the notation by a simple example. Let us go back to the basic setting in which we formulate the temperature at each step with a pair $(t_0,\alpha)$. In this case, $\feasible = \mathbb{R}^{+} \times (0,1/m) $ would be the set of all valid parameters. Moreover, a natural example for $\score$ is the probability of finding a correct solution after a given (say $m$) number of steps. This way, the problem is to find a temperature $t_0$ and a cooling rate $\alpha$ that maximize the success probability after performing $m$ moves of \textsf{SA}. Our attention in this work is focused on an even more general setting. We denote the cooling schedule by a sequence of non-increasing temperatures $\langle t_1, t_2, \ldots, t_m\rangle$ for a fixed $m$. Thus, in our setting we have $\feasible \subseteq (\mathbb{R}^{+})^m$ subject to the temperatures being non-increasing. For simplicity, and without loss of generality, we assume that the energies of the nodes are integer numbers in range $\{1,2,\ldots,\emax\}$.

Any \textsf{SA} algorithm basically makes a random walk on a search graph in which every node represents a potential/partial solution
for the problem. For instance, when the underlying problem is \textsf{SAT}, every node of the search graph is a true/false assignment to the variables of the program. The energy of each node is a local guess on how well that solution satisfies the goals of the problem. For the case of \textsf{SAT} for instance, one simple energy function for a node is the number of clauses the corresponding solution satisfies. In our setting, we make no assumption on the energy of the nodes though in practice we expect that a higher energy signals a better solution. Since an \textsf{SA} algorithm makes a random walk, its state at every step can be shown via a distribution over the nodes of the search graph. Initially, this distribution shows the likelihood of each node being used as the starting solution and as the algorithm proceeds, the distribution changes based on the criteria of the random walk. The final state of the algorithm represents the likelihood of each node reported as the final solution. Thus, we wish the final distribution of our algorithm to be highly concentrated on the solution nodes.

We evaluate our learning algorithm based on two quantities: \textit{sample complexity} and \textit{simulation complexity}\footnote{This is equivalent to the notion of running time if we assume that our \textsf{SA} procedure halts after a polynomial number of steps.}. The former measures the number of samples one needs in order to find an (approximately) optimal cooling schedule and the latter measures the runtime of the learning algorithm in order to find an optimal cooling schedule.

\subsection{Our Results and Techniques}
Our main results are concerned with the sample complexity of the learning problem. As a typical challenge for learning problems, we have to face the issue that the space of the problem is infinitely large as there are infinitely many cooling schedules for an \textsf{SA} algorithm. In order to prove a bound on the sample complexity, the first step is to show that by losing a small additive error, we can bound the space of the solutions to a finite set. We begin by explaining this in Section \ref{sec:discretization}.\\[0.1cm]

\noindent \textbf{Step 1: From infinity to linear:} In this step, we show that although the space of the problem is infinitely large, only a polynomial number of samples suffice to approximate the optimal solution within desirable guarantees. This step is quite classic as discretization is the typical approach to bound the solution set.

Recall that $m$ is the length of the random walk in the search graph. We show in Section \ref{sec:discretization} that $\tilde O(m)$ samples from the distribution suffice to approximate the optimal cooling schedule within a small additive error. Notice that $\tilde O$ hides the polylogarithmic factors (both in terms of $m$ and $\emax$). This basically gives an almost linear upper bound on sample complexity based on the number of steps of the algorithm.

Roughly speaking, the total number of samples we need in order to approximate the optimal cooling schedule is logarithmic in terms of the number of \textit{candidate solutions} we have. Initially, the space of cooling schedules is infinitely large, however, a discretization technique can reduce the space of candidate solutions to $2^{\tilde O(m)}$ many. More precisely, we define a discretized temperature set $T$ whose size is $\tilde O(m)$ and show that there is an almost optimal solution that only uses the temperatures in $T$. This reduces the space of candidate solutions to $(O(m))^m \leq 2^{\tilde O(m)}$ which implies that the sample complexity is bounded by $\tilde O(m)$. 

The only non-trivial part of the above analysis is to show that a discretized set of temperatures with size $\tilde O(m)$ is enough to approximate the optimal cooling schedule within an arbitrarily small additive error. Let us fix an $\epsilon > 0$ and assume that the goal is to construct a discretized set of temperatures $T$ such that there is a cooling schedule that only uses the temperatures of $T$ and its score is at most $\epsilon$ smaller than the optimal solution. One convenient way to construct such a set is to make sure for each $t > 0$ there is a $t' \in T$ such that for any $1 \leq x \leq \emax$ we have $|e^{-x/t} - e^{-x/t'}| \leq \epsilon/m$. Then we can imply that if we replace every temperature $t_i$ of the optimal solution with its corresponding $t'_i$ of the discretized set, each step we make a different decision with probability at most $\epsilon /m$ and thus the total error is bounded by $\epsilon$. That is, with probability $1-\epsilon$ our algorithm traverses the exact same path as had we not modified the optimal cooling schedule. It is not hard to prove that such a condition can be met by having $O(m \log \emax)$ elements in $|T|$ which gives us an almost linear bound on the sample complexity. 


\vspace{0.2cm}
{\noindent \textbf{Theorem}~\ref{thm:sample}, [restated informally]. \textit{For any $\epsilon > 0$, the sample complexity of approximating the learning problem within an additive error of $\epsilon$ is bounded by $\tilde O_{\epsilon}(m)$.\\}}

Up to this point we show that an almost linear number of queries is sufficient for approximating an optimal cooling schedule. This raises two questions: i) Can we improve the bound such that the dependence on $m$ is subpolynomial? In particular, do polylogarithmically many samples suffice for our purpose? ii) If the answer to the first question is negative, can we prove a linear lower bound on the sample complexity? As we show in the following, the answer to both questions is negative!\\[0.3cm]

\noindent \textbf{Step 2: A polynomial lower bound:} We present a negative answer to the first question. Although this step gives us a lower bound, our improved upper bound is actually inspired by this lower bound. The first attempt to prove a lower bound is to understand the limit of the discretization technique explained above. Therefore, we ask the following question: ``assuming that our algorithm first constructs a discretized set of temperatures and then seeks to find an optimal solution that only uses the discretized temperatures, how many samples do we need?" Indeed, the answer to this question does not imply a lower bound in general, but it does give us an insight into the problem which leads to a general lower bound.

To answer the above question, we need to understand what is the smallest set $T$ of temperatures that can be used to make a cooling schedule whose score is very close to the optimal solution? The search graph shown in Figure \ref{fig:lowera} proves that $|T|$ should be at least as large as $\tilde \Omega(\sqrt{m})$, otherwise the guarantee may not hold.

In the search graph of Figure \ref{fig:lowera},  we set $m' = m/100$. For a fixed $\tau$, we set $x$ in a way that $e^{-x/\tau} = 1/2$, that is if the temperature is equal to $\tau$ the probability of making a downhill move is exactly equal to $1/2$\footnote{For now, we assume $x$ can be an arbitrary real number but this comes without loss of generality.}. The goal of this search graph is to start the \textsf{SA} algorithm from the initial node and the only acceptable solution node is the final node. 
\input{figs/lower2a}
The search graph of Figure \ref{fig:lowera} is particularly interesting because of the following observations: i) A cooling schedule of length $m$ only having temperature $\tau$ is guaranteed to reach the final node with high probability. ii) A cooling schedule of length $m$ that does not contain any temperature in range $[\tau(1-\tilde \Omega(m^{-1/2})),\tau(1+\tilde \Omega(m^{-1/2}))]$ has very little chance to reach the final node. As a consequence, if the multiplicative distance between two consecutive temperatures in our discretized set is more than $1+ \tilde \Omega(m^{-1/2})$, one can delicately design such a search graph for which our discretization performs poorly while the optimal solution gets a score close to $1$. This implies that the size of the discretized set has to be at least $\tilde \Omega(\sqrt{m})$ to prove a bound.

While the above argument shows that our specific algorithm definitely needs $\tilde O(\sqrt{n})$ samples\footnote{See the proof of Theorem \ref{theorem:main} for more details.}, it does not give a lower bound in general. However, we show in Section \ref{sec:improved-upper} with a slightly more advanced analysis that any algorithm requires at least $\tilde \Omega(m^{1/3})$ samples from the distribution in order to guarantee a non-trivial bound. While the heart of the proof is based on the same search graph, in order to extend the observation to all algorithms, we slightly lose on the exponent of $m$ in the lower bound.

\vspace{0.2cm}
{\noindent \textbf{Theorem}~\ref{theorem:lower}, [restated informally]. \textit{Any learning algorithm that approximates the solution within an additive error of $0.5$ needs at least $\tilde \Omega(m^{1/3})$ samples from the distribution.}}

Before proceeding to the third step, we would like to note an implication of this result in the context of simulated annealing. There have been several attempts in the literature to understand the complexity of simulated annealing. One question asked in the literature from both theoretical and practical standpoints is if there is a meaningful difference between Simulated Annealing and the Metropolis Algorithm~\cite{wegener2005simulated,hajek1988cooling}.  Metropolis is a special case of simulated annealing where the temperature does not change by time. That is the cooling schedule repeats a single temperature $m$ times. While this observation was made previously, our lower bound also implies that (from a theoretical standpoint) there is a meaningful difference between the two algorithms as Metropolis can be learned with much fewer samples which shows there are cases for which simulated annealing performs much better. Another example is when the temperature drops linearly for which the sample complexity is small. More generally, this lower bound actually shows a gap between \textsf{SA} and any special case of \textsf{SA} whose cooling schedule has complexity smaller than $m^{1/3}$. For instance, it shows that an extended version of Metropolis that uses $m^{1/3-\epsilon}$ many different temperatures in the cooling schedule is not competitive with the general \textsf{SA} algorithm.\\[0.1cm]

\noindent \textbf{Step 3: From linear to sublinear:} Perhaps the more surprising result of this paper is that the sample complexity can be improved to $\tilde O(\sqrt{m})$. Our algorithm is almost identical to the one explained in Step 1 except that we construct a smaller set $T$ whose size is bounded by $\tilde O(\sqrt{m})$. Then we argue that the total number of cooling schedules with this temprature set is bounded by $2^{\tilde O(\sqrt{m})}$ which leads to sample complexity $\tilde O(\sqrt{m})$.

The first pointer to this result is that there is no clear way to improve the lower bound of Step 2. Keep in mind that for the lower bound, we construct a search graph for which a particular cooling schedule works well, but if we multiply (or divide) each temperature by a small factor $1+\tilde \Omega(m^{-1/2})$, the score of the algorithm drops significantly. Obviously, if one comes up with a better search graph for which a multiplicative factor of $1+\tilde \Omega(m^{-1/2-\epsilon})$ breaks the solution, then it shows that it is impossible to obtain an upper bound of $\tilde O(\sqrt{m})$ with the discretization technique. Failure to make a better bad instance brings us to the possibility that maybe massaging each temperature by a multiplicative factor of $1+\tilde O(m^{-1/2})$ cannot hurt the score of the cooling schedule significantly. We show that this is indeed the case! 

Recall that in Step 1, in order to prove that the discretized cooling schedules perform almost optimally, we show that there is a discretized cooling schedule that behaves the sames as the optimal cooling schedule with probability $1-\epsilon$. That is, in the unlikely event of making a different decision (we call it a \textit{mistake}) we give 0 credit to our discretized cooling schedule, yet we prove that the score is pretty close to that of the optimal. Clearly, this is a loose upper bound as we do not expect to lose too much by making a single mistake.

We illustrate the idea with a toy problem. Consider a complete binary tree of depth $m$. The root has depth $0$ and the leaves have depth $m$. Each leaf is attributed to a score which is either 0 or 1. The score of each non-leaf node is the average of the scores of its children. In other words, if we make a random walk towards the leaves with equal probability of going to each child, the score of a node is equal to the probability of reaching a leaf with score 1 using the random-walk. Let us call this \textit{even-random-walk} and consider a different type of random-walk, namely \textit{uneven-random-walk}. The uneven-random-walk is pretty much the same as the even random walk, except that at some depth $i$ uniformly drawn from $[1,m]$, an adversary may change the decision of which child to go to. The toy-problem is to understand how much the score of a node hurts by replacing even-random-walk by uneven-random-walk.

To study this, we attribute to each node a \textit{deviation value} which is equal to the absolute value of the difference between the scores of it children. This roughly captures an upper bound of the score we lose, if we traverse the edges of that node with a different criteria (other than $1/2,1/2$). Thus, we need to know what is the average deviation values of the nodes in an even-random-walk? This roughly tells us how much we lose in the score, if an adversary changes the criteria of the walk at some random point!

The upper bound on the answer is $O(1/\sqrt{m})$ no matter how the leaves are scored. It goes beyond the scope of this paper, but we mention the idea in the hope that it helps a mindful reader decipher some of steps that we take in the proof of Lemma \ref{lemma:challenge}. Define a deviation function $f(x):[0,1] \rightarrow [0,0.25] = x - x^2$. One can show by induction that starting from each node $v$ of depth $i$, the average sum of deviations in a random walk is bounded by $O((f(s_v) + (m-i)/m)\sqrt{m})$ where $s_v$ is the score of node $v$ (obtained via even-random-walk). 

The toy problem illustrates that in the event that our optimal solution makes decisions with probability $1/2,1/2$ (which is indeed the case for our lower bound), we can afford to make $O(\sqrt{m})$ mistakes and not lose much in the average score. This does not hold if the decisions are made with different probabilities. To see this, consider the case that only the rightmost leaf has a score 1 and the rest of the leaves have scores 0. Moreover, the probability of going to the right child in the random walk is $1-\epsilon/m$ and the probability of going to the left child is $\epsilon/m$. In this case, the average  deviation is $\Omega(1)$ when we start from the root and make a random walk according to the probabilities.

The next observation is that when the decisions are not necessarily $1/2,1/2$ say $p,1-p$, multiplying the temperature by a factor of $1+x$ changes the probabilities by at most $\max\{\ln 1/p,1\} \min\{p,1-p\} x$ (see Observation \ref{obs:tavan2}). That is, as the probabilities deviate from $1/2$, the probability of making a ``mistake" drops linearly. More precisely, the multiplicative term $\min\{p,1-p\}$ gives us extra power to deal with these situations. For instance, if $p < 1/\sqrt{m}$ or $p > 1-1/\sqrt{m}$ then the probabilities change by an additive error of $\tilde O(1/m)$ when we multiply the temperature by a factor of $1+\tilde O(m^{-1/2})$. This error is tolerable since we can afford to have an error of $\epsilon / m$ for each decision we make.

The proof is based on the above ideas but the analysis is quite involved and rather cryptic by nature. We show in Section \ref{sec:improved-upper} that if the temperatures in the discretized set are at most $1+\tilde O(m^{-1/2})$ away from each other (multiplicative), then one can make a cooling schedule by the discretized temperatures whose score is arbitrarily close to that of the optimal solution. This then can be used to obtain an upper bound of $\tilde O(\sqrt{m})$ on the sample complexity of the problem.

\vspace{0.2cm}
{\noindent \textbf{Theorem}~\ref{theorem:main}, [restated informally]. \textit{For any $\epsilon > 0$, the sample complexity of approximating the learning problem within an additive error of $\epsilon$ is bounded by $\tilde O_{\epsilon}(\sqrt{m})$.} \\}

The second part of the paper is concerned with the computational aspects of the learning problem. Although we prove that the sample complexity is polynomial without any assumptions, it seems that extra assumptions are necessary for the runtime concerns. Notice that we make no assumption on the underlying problem and the only information available to us when we sample an instance of the problem is a huge search graph containing exponentially many vertices. Even if we bring the underling problem into the setting, it is not clear how we can make use of the conditions of a problem such as \textsf{SAT} to find the right cooling schedule. Keep in mind that the complexity of the underlying problem is the reason we use simulated annealing in the first place. Therefore, we introduce a stylized model to make the problem more tractable. We call our model \textit{the monotone stationary graph}. Although the model relies on extra assumptions, it features nice properties that make it particularly suitable for our purpose.

 First, it gives a compact representation for every instance of the problem. Up to this point, we treated each problem instance as a huge search graph with exponentially many vertices which is too big to store in the memory let alone optimizing the solution over it. Our model represents the search graphs in a more efficient way. Next, notice that even if we fix a well-defined representation for a search graph, one should be able to recover the new representation of a problem instance without spending too much time (and of course without taking a complete look at the already exponentially large search graph). Our model makes it possible to recover the stationary graph in polynomial time. Finally, the any model used for our problem has to give us enough structure so that finding an approximately optimal cooling schedule becomes polynomially tractable in the new setting. This is the most important feature of our model.

In our model, we represent each instance of the problem as a graph. Vertices of this graph correspond to the temperatures in our discretized set. Intuitively, for a temperature $t \in T$, its corresponding vertex in the graph represent the state of an \textsf{SA} algorithm that runs infinitely many steps with temperature $t$. Thus, when the state of our algorithm is close to such a stationary distribution, we assume that our algorithm is pointing at the corresponding vertex in the monotone stationary graph. We draw edges between the vertices to specify how many steps we need to take in the \textsf{SA} algorithm to move between the stationary distributions. Since in our model, the state of an algorithm can be approximated with a node in this graph, we can also determine its score by examining the corresponding stationary distribution.

Therefore, given $n$ instances of the underlying problem, our goal is to find a cooling schedule that obtains the highest average score for these instances by our model. We consider the following three settings and provide a solution for each one of them: 1) \textsf{identical-paths}: in this setting, we assume that the optimal cooling schedule traverses the same path for all $n$ instances. 2) \textsf{separate-paths}: in this setting, we allow the optimal solution to use different paths for different instances. 3) \textsf{separate-paths + all-satisfied:} This is a special case of the second setting where we know that there exists a cooling schedule of length $m$ that is optimal for all instances and brings us to the last node for each monotone stationary graph.

To obtain polynomial time solutions, we introduce the notion of an $(\alpha,\epsilon)$-approximate cooling schedule. In such a solution we allow the cooling schedule to violate the size constraint by a factor of $\alpha$ with the promise that its score is no more than $\epsilon$ smaller than the score of the optimal cooling schedule of length $m$. With this notation, we present computational results shown in Table \ref{table:resultso}. 

We assume throughout this paper that the scores improves as energy \textit{increases}. Also, a downhill move is a move which hurts the energy of a node and thus is accepted with some probability smaller than $1$. However, whenever the score of a node does not hurt in a move, such a move is always made.
\input{tables/results}


%% file: figs/superficial.tex
\begin{figure}
\centering
	\usetikzlibrary{arrows,calc,intersections}
	\usetikzlibrary{calc}

	\tikzset{fontscale/.style = {font=\relsize{#1}}
	}

		
		\def\r{4}
		\def\n{8} \def\myangles{{90,110,130,150,230,250,270,50,70}}
		\newcounter{np} \pgfmathsetcounter{np}{\n+1}
		\newcounter{na} \newcounter{nb} \newcounter{nc}
		\newcounter{ia} 
		\pgfmathsetcounter{na}{\n-1}    
		\pgfmathsetcounter{nb}{\n-2}    
		\pgfmathsetcounter{nc}{\n-3}    
		\newcounter{q} \setcounter{q}{0}    
		\newcounter{e} \setcounter{e}{0}    
		\newcounter{a} \setcounter{a}{0}    
		\newcounter{b} \setcounter{b}{1}    
		\newcounter{c} \setcounter{c}{2}    
		\newcounter{d} \setcounter{d}{2}    

\begin{tikzpicture}[thick,scale=0.8, every node/.style={scale=0.8}]
    \fill[fill=blue!10!green!10!,draw=blue,dotted,thick] (0,0) circle (\r);
    \pgfmathparse{\n-1} \let\nn\pgfmathresult 
    \foreach \i in {0,...,\nn}{
        \pgfmathparse{\i+1} \let\ii\pgfmathresult
        \pgfmathparse{\myangles[\i]} \let\t\pgfmathresult
        \foreach \j in {\ii,...,\n}
            \pgfmathparse{\myangles[\j]} \let\u\pgfmathresult
            \draw[blue,very thick] ({\r*cos(\t)},{\r*sin(\t)})--({\r*cos(\u)},{\r*sin(\u)});
        }
    \foreach \i in {0,...,3}{
    	\pgfmathparse{\myangles[6]}    \let\t\pgfmathresult
    	\draw[blue,very thick] ({\r*cos(\t)+2+2*\i},{\r*sin(\t)})--({\r*cos(\t)+2*\i},{\r*sin(\t)});
    }
     \foreach \i in {0,...,3}{
    	\pgfmathparse{\myangles[\i]}    \let\t\pgfmathresult
    	\pgfmathsetcounter{ia}{\i+1}
    	\fill[draw=blue,fill=blue!20!,thick]
    	({\r*cos(\t)},{\r*sin(\t)})circle (4mm) node{$\mathbf{\theia}$};
    }

\tikzset{
	font={\fontsize{100pt}{100}\selectfont}}

     \node[rotate=12] at (-3.5,-0.4)  { \Huge $\vdots$};
     \node[rotate=-18] at (2.8,-0.9)  { \Huge $\vdots$};
     \node at (10,-4)  { \Huge $\ldots$};
 
 \tikzset{
     	font={\fontsize{5pt}{12}\selectfont}}
     
     \foreach \i in {4,...,5}{
	\pgfmathparse{\myangles[\i]}    \let\t\pgfmathresult
	\pgfmathsetcounter{ia}{6-\i}
	\fill[draw=blue,fill=blue!20!,thick]
	({\r*cos(\t)},{\r*sin(\t)})circle (4mm) node{$n-\mathbf{\theia}$};
}
\tikzset{
	font={\fontsize{10pt}{12}\selectfont}}

     \foreach \i in {6,...,6}{
	\pgfmathparse{\myangles[\i]}    \let\t\pgfmathresult
	\pgfmathsetcounter{ia}{6-\i}
	\fill[draw=blue,fill=blue!20!,thick]
	({\r*cos(\t)},{\r*sin(\t)})circle (4mm) node{$n$};
}


\tikzset{
	font={\fontsize{5pt}{12}\selectfont}}

     \foreach \i in {0,...,3}{
	\pgfmathparse{\myangles[6]}    \let\t\pgfmathresult
	\pgfmathsetcounter{ia}{\i+1}
	\fill[draw=blue,fill=blue!20!,thick]
	({\r*cos(\t)+2+2*\i},{\r*sin(\t)})circle (4mm) node{$2n+\theia$};
}

\tikzset{
	font={\fontsize{10pt}{12}\selectfont}}

     \foreach \i in {5,...,5}{
	\pgfmathparse{\myangles[6]}    \let\t\pgfmathresult
	\pgfmathsetcounter{ia}{\i+1}
	\fill[draw=blue,fill=red!20!,thick]
	({\r*cos(\t)+2+2*\i},{\r*sin(\t)})circle (4mm) node{$3n$};
}
\tikzset{
	font={\fontsize{5pt}{12}\selectfont}}

     \foreach \i in {7,...,7}{
	\pgfmathparse{\myangles[\i]}    \let\t\pgfmathresult
	\pgfmathsetcounter{ia}{8-\i}
	\fill[draw=blue,fill=blue!20!,thick]
	({\r*cos(\t)},{\r*sin(\t)})circle (4mm) node{$2n-\mathbf{\theia}$};
}

\tikzset{
	font={\fontsize{10pt}{12}\selectfont}}

     \foreach \i in {8,...,8}{
	\pgfmathparse{\myangles[\i]}    \let\t\pgfmathresult
	\pgfmathsetcounter{ia}{6-\i}
	\fill[draw=blue,fill=green,thick]
	({\r*cos(\t)},{\r*sin(\t)})circle (4mm) node{$2n$};
}
        %
    \whiledo{\theq=0}{ 
        \stepcounter{e}
        \ifthenelse{\thee=1000}{\setcounter{q}{1}}{}
        \ifthenelse{\thed=\n}
            {\ifthenelse{\thec=\thena}
                {\ifthenelse{\theb=\thenb}
                    {\ifthenelse{\thea=\thenc}
                        {\setcounter{q}{1}}
                        {   \stepcounter{a}
                            \pgfmathsetcounter{b}{\thea+1}
                            \pgfmathsetcounter{c}{\thea+2}
                            \pgfmathsetcounter{d}{\thea+3}                  
                        }
                    }
                    {   \stepcounter{b}
                        \pgfmathsetcounter{c}{\theb+1}
                        \pgfmathsetcounter{d}{\theb+2}                      
                    }
                }
                {   \stepcounter{c}
                    \pgfmathsetcounter{d}{\thec+1}
                }
            }
            {\stepcounter{d}}
        \ifthenelse{\theq=0}{
            \pgfmathparse{\r*cos(\myangles[\thea])} \let\xa\pgfmathresult
            \pgfmathparse{\r*sin(\myangles[\thea])} \let\ya\pgfmathresult
            \pgfmathparse{\r*cos(\myangles[\theb])} \let\xb\pgfmathresult
            \pgfmathparse{\r*sin(\myangles[\theb])} \let\yb\pgfmathresult
            \pgfmathparse{\r*cos(\myangles[\thec])} \let\xc\pgfmathresult
            \pgfmathparse{\r*sin(\myangles[\thec])} \let\yc\pgfmathresult
            \pgfmathparse{\r*cos(\myangles[\thed])} \let\xd\pgfmathresult
            \pgfmathparse{\r*sin(\myangles[\thed])} \let\yd\pgfmathresult
            \coordinate  (A) at (\xa,\ya);
            \coordinate  (B) at (\xb,\yb);
            \coordinate  (C) at (\xc,\yc);
            \coordinate  (D) at (\xd,\yd);
            \path[name path=sega] (A) -- (C);
            \path[name path=segb] (B) -- (D);
            \path [name intersections={of=sega and segb}];
            \coordinate (X) at (intersection-1);
            }{}
    }
\end{tikzpicture}
\tikzset{
	font={\fontsize{10pt}{12}\selectfont}}

\addtocounter{e}{-1}
\caption{A search graph with $3n$ nodes is illustrated in this figure.  The numbers on the nodes show their energy or in other words goodness of the nodes. In this example, we consider the red node to be the only solution of the problem.}
\end{figure}
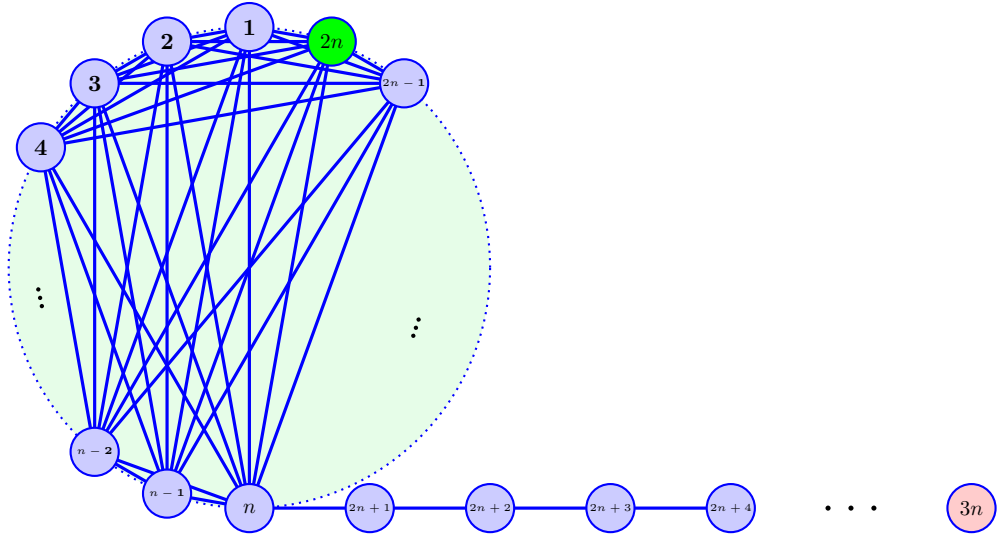\label{fig:superficial}

%% file: figs/lower2a.tex
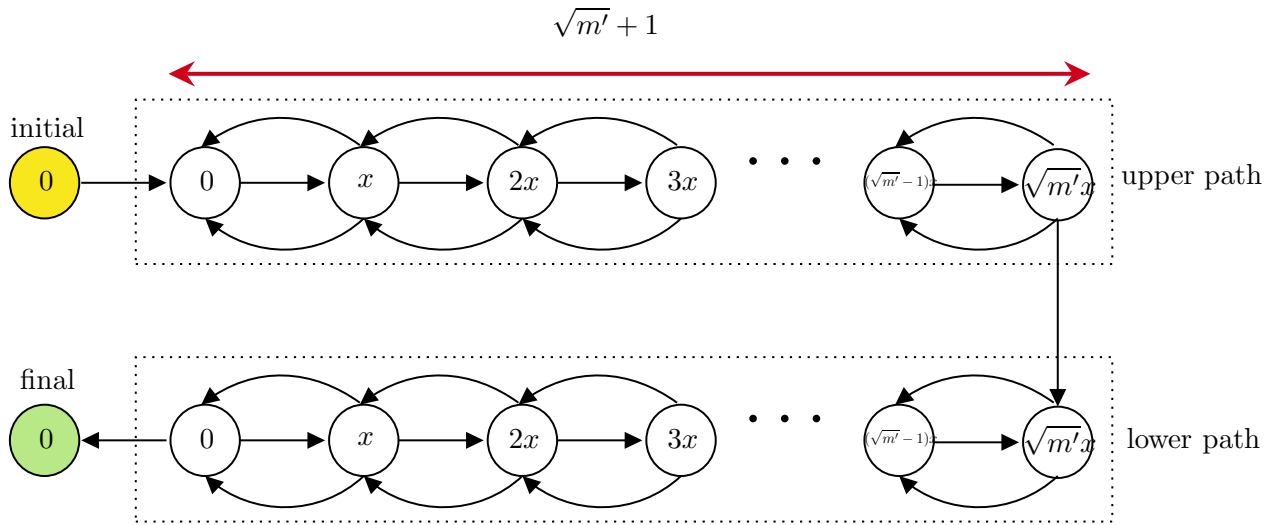
\begin{figure}[H]\centering

\tikzset{every picture/.style={line width=0.75pt}} 

\begin{tikzpicture}[x=0.75pt,y=0.75pt,yscale=-1,xscale=1]

\draw   (106,96) .. controls (106,86.06) and (113.84,78) .. (123.5,78) .. controls (133.16,78) and (141,86.06) .. (141,96) .. controls (141,105.94) and (133.16,114) .. (123.5,114) .. controls (113.84,114) and (106,105.94) .. (106,96) -- cycle ;
\draw    (141,96) -- (182,96) ;
\draw [shift={(184,96)}, rotate = 180] [fill={rgb, 255:red, 0; green, 0; blue, 0 }  ][line width=0.75]  [draw opacity=0] (8.93,-4.29) -- (0,0) -- (8.93,4.29) -- cycle    ;

\draw   (186,96) .. controls (186,86.06) and (193.84,78) .. (203.5,78) .. controls (213.16,78) and (221,86.06) .. (221,96) .. controls (221,105.94) and (213.16,114) .. (203.5,114) .. controls (193.84,114) and (186,105.94) .. (186,96) -- cycle ;
\draw    (221,96) -- (262,96) ;
\draw [shift={(264,96)}, rotate = 180] [fill={rgb, 255:red, 0; green, 0; blue, 0 }  ][line width=0.75]  [draw opacity=0] (8.93,-4.29) -- (0,0) -- (8.93,4.29) -- cycle    ;

\draw   (266,96) .. controls (266,86.06) and (273.84,78) .. (283.5,78) .. controls (293.16,78) and (301,86.06) .. (301,96) .. controls (301,105.94) and (293.16,114) .. (283.5,114) .. controls (273.84,114) and (266,105.94) .. (266,96) -- cycle ;
\draw    (301,96) -- (342,96) ;
\draw [shift={(344,96)}, rotate = 180] [fill={rgb, 255:red, 0; green, 0; blue, 0 }  ][line width=0.75]  [draw opacity=0] (8.93,-4.29) -- (0,0) -- (8.93,4.29) -- cycle    ;

\draw    (491,97) -- (532,97) ;
\draw [shift={(534,97)}, rotate = 180] [fill={rgb, 255:red, 0; green, 0; blue, 0 }  ][line width=0.75]  [draw opacity=0] (8.93,-4.29) -- (0,0) -- (8.93,4.29) -- cycle    ;

\draw   (536,97) .. controls (536,87.06) and (543.84,79) .. (553.5,79) .. controls (563.16,79) and (571,87.06) .. (571,97) .. controls (571,106.94) and (563.16,115) .. (553.5,115) .. controls (543.84,115) and (536,106.94) .. (536,97) -- cycle ;
\draw    (203.5,114) .. controls (181.45,134.58) and (147.39,134.99) .. (124.86,115.23) ;
\draw [shift={(123.5,114)}, rotate = 403.03] [fill={rgb, 255:red, 0; green, 0; blue, 0 }  ][line width=0.75]  [draw opacity=0] (8.93,-4.29) -- (0,0) -- (8.93,4.29) -- cycle    ;

\draw [color={rgb, 255:red, 208; green, 2; blue, 27 }  ,draw opacity=1 ][line width=1.5]    (108,41) -- (567,41) ;
\draw [shift={(570,41)}, rotate = 180] [fill={rgb, 255:red, 208; green, 2; blue, 27 }  ,fill opacity=1 ][line width=1.5]  [draw opacity=0] (13.4,-6.43) -- (0,0) -- (13.4,6.44) -- (8.9,0) -- cycle    ;
\draw [shift={(105,41)}, rotate = 0] [fill={rgb, 255:red, 208; green, 2; blue, 27 }  ,fill opacity=1 ][line width=1.5]  [draw opacity=0] (13.4,-6.43) -- (0,0) -- (13.4,6.44) -- (8.9,0) -- cycle    ;
\draw    (201.5,77) .. controls (178.35,60.25) and (148.9,57.09) .. (122.69,77.07) ;
\draw [shift={(121.5,78)}, rotate = 321.6] [fill={rgb, 255:red, 0; green, 0; blue, 0 }  ][line width=0.75]  [draw opacity=0] (8.93,-4.29) -- (0,0) -- (8.93,4.29) -- cycle    ;

\draw    (553.5,115) -- (553.5,207) ;
\draw [shift={(553.5,209)}, rotate = 270] [fill={rgb, 255:red, 0; green, 0; blue, 0 }  ][line width=0.75]  [draw opacity=0] (8.93,-4.29) -- (0,0) -- (8.93,4.29) -- cycle    ;

\draw   (346,96) .. controls (346,86.06) and (353.84,78) .. (363.5,78) .. controls (373.16,78) and (381,86.06) .. (381,96) .. controls (381,105.94) and (373.16,114) .. (363.5,114) .. controls (353.84,114) and (346,105.94) .. (346,96) -- cycle ;
\draw   (456,96) .. controls (456,86.06) and (463.84,78) .. (473.5,78) .. controls (483.16,78) and (491,86.06) .. (491,96) .. controls (491,105.94) and (483.16,114) .. (473.5,114) .. controls (463.84,114) and (456,105.94) .. (456,96) -- cycle ;
\draw    (283.5,114) .. controls (261.45,134.58) and (227.39,134.99) .. (204.86,115.23) ;
\draw [shift={(203.5,114)}, rotate = 403.03] [fill={rgb, 255:red, 0; green, 0; blue, 0 }  ][line width=0.75]  [draw opacity=0] (8.93,-4.29) -- (0,0) -- (8.93,4.29) -- cycle    ;

\draw    (363.5,114) .. controls (341.45,134.58) and (307.4,134.99) .. (284.86,115.23) ;
\draw [shift={(283.5,114)}, rotate = 403.03] [fill={rgb, 255:red, 0; green, 0; blue, 0 }  ][line width=0.75]  [draw opacity=0] (8.93,-4.29) -- (0,0) -- (8.93,4.29) -- cycle    ;

\draw    (553.5,114) .. controls (531.45,134.58) and (497.4,134.99) .. (474.86,115.23) ;
\draw [shift={(473.5,114)}, rotate = 403.03] [fill={rgb, 255:red, 0; green, 0; blue, 0 }  ][line width=0.75]  [draw opacity=0] (8.93,-4.29) -- (0,0) -- (8.93,4.29) -- cycle    ;

\draw    (281.5,77) .. controls (258.35,60.25) and (228.9,57.09) .. (202.69,77.07) ;
\draw [shift={(201.5,78)}, rotate = 321.6] [fill={rgb, 255:red, 0; green, 0; blue, 0 }  ][line width=0.75]  [draw opacity=0] (8.93,-4.29) -- (0,0) -- (8.93,4.29) -- cycle    ;

\draw    (361.5,77) .. controls (338.35,60.25) and (308.9,57.09) .. (282.69,77.07) ;
\draw [shift={(281.5,78)}, rotate = 321.6] [fill={rgb, 255:red, 0; green, 0; blue, 0 }  ][line width=0.75]  [draw opacity=0] (8.93,-4.29) -- (0,0) -- (8.93,4.29) -- cycle    ;

\draw    (551.5,77) .. controls (528.35,60.25) and (498.9,57.09) .. (472.69,77.07) ;
\draw [shift={(471.5,78)}, rotate = 321.6] [fill={rgb, 255:red, 0; green, 0; blue, 0 }  ][line width=0.75]  [draw opacity=0] (8.93,-4.29) -- (0,0) -- (8.93,4.29) -- cycle    ;

\draw   (106,226) .. controls (106,216.06) and (113.84,208) .. (123.5,208) .. controls (133.16,208) and (141,216.06) .. (141,226) .. controls (141,235.94) and (133.16,244) .. (123.5,244) .. controls (113.84,244) and (106,235.94) .. (106,226) -- cycle ;
\draw    (141,226) -- (182,226) ;
\draw [shift={(184,226)}, rotate = 180] [fill={rgb, 255:red, 0; green, 0; blue, 0 }  ][line width=0.75]  [draw opacity=0] (8.93,-4.29) -- (0,0) -- (8.93,4.29) -- cycle    ;

\draw   (186,226) .. controls (186,216.06) and (193.84,208) .. (203.5,208) .. controls (213.16,208) and (221,216.06) .. (221,226) .. controls (221,235.94) and (213.16,244) .. (203.5,244) .. controls (193.84,244) and (186,235.94) .. (186,226) -- cycle ;
\draw    (221,226) -- (262,226) ;
\draw [shift={(264,226)}, rotate = 180] [fill={rgb, 255:red, 0; green, 0; blue, 0 }  ][line width=0.75]  [draw opacity=0] (8.93,-4.29) -- (0,0) -- (8.93,4.29) -- cycle    ;

\draw   (266,226) .. controls (266,216.06) and (273.84,208) .. (283.5,208) .. controls (293.16,208) and (301,216.06) .. (301,226) .. controls (301,235.94) and (293.16,244) .. (283.5,244) .. controls (273.84,244) and (266,235.94) .. (266,226) -- cycle ;
\draw    (301,226) -- (342,226) ;
\draw [shift={(344,226)}, rotate = 180] [fill={rgb, 255:red, 0; green, 0; blue, 0 }  ][line width=0.75]  [draw opacity=0] (8.93,-4.29) -- (0,0) -- (8.93,4.29) -- cycle    ;

\draw    (491,227) -- (532,227) ;
\draw [shift={(534,227)}, rotate = 180] [fill={rgb, 255:red, 0; green, 0; blue, 0 }  ][line width=0.75]  [draw opacity=0] (8.93,-4.29) -- (0,0) -- (8.93,4.29) -- cycle    ;

\draw   (536,227) .. controls (536,217.06) and (543.84,209) .. (553.5,209) .. controls (563.16,209) and (571,217.06) .. (571,227) .. controls (571,236.94) and (563.16,245) .. (553.5,245) .. controls (543.84,245) and (536,236.94) .. (536,227) -- cycle ;
\draw    (203.5,244) .. controls (181.45,264.58) and (147.39,264.99) .. (124.86,245.23) ;
\draw [shift={(123.5,244)}, rotate = 403.03] [fill={rgb, 255:red, 0; green, 0; blue, 0 }  ][line width=0.75]  [draw opacity=0] (8.93,-4.29) -- (0,0) -- (8.93,4.29) -- cycle    ;

\draw    (201.5,207) .. controls (178.35,190.26) and (148.9,187.09) .. (122.69,207.07) ;
\draw [shift={(121.5,208)}, rotate = 321.6] [fill={rgb, 255:red, 0; green, 0; blue, 0 }  ][line width=0.75]  [draw opacity=0] (8.93,-4.29) -- (0,0) -- (8.93,4.29) -- cycle    ;

\draw   (346,226) .. controls (346,216.06) and (353.84,208) .. (363.5,208) .. controls (373.16,208) and (381,216.06) .. (381,226) .. controls (381,235.94) and (373.16,244) .. (363.5,244) .. controls (353.84,244) and (346,235.94) .. (346,226) -- cycle ;
\draw   (456,226) .. controls (456,216.06) and (463.84,208) .. (473.5,208) .. controls (483.16,208) and (491,216.06) .. (491,226) .. controls (491,235.94) and (483.16,244) .. (473.5,244) .. controls (463.84,244) and (456,235.94) .. (456,226) -- cycle ;
\draw    (283.5,244) .. controls (261.45,264.58) and (227.39,264.99) .. (204.86,245.23) ;
\draw [shift={(203.5,244)}, rotate = 403.03] [fill={rgb, 255:red, 0; green, 0; blue, 0 }  ][line width=0.75]  [draw opacity=0] (8.93,-4.29) -- (0,0) -- (8.93,4.29) -- cycle    ;

\draw    (363.5,244) .. controls (341.45,264.58) and (307.4,264.99) .. (284.86,245.23) ;
\draw [shift={(283.5,244)}, rotate = 403.03] [fill={rgb, 255:red, 0; green, 0; blue, 0 }  ][line width=0.75]  [draw opacity=0] (8.93,-4.29) -- (0,0) -- (8.93,4.29) -- cycle    ;

\draw    (553.5,244) .. controls (531.45,264.58) and (497.4,264.99) .. (474.86,245.23) ;
\draw [shift={(473.5,244)}, rotate = 403.03] [fill={rgb, 255:red, 0; green, 0; blue, 0 }  ][line width=0.75]  [draw opacity=0] (8.93,-4.29) -- (0,0) -- (8.93,4.29) -- cycle    ;

\draw    (281.5,207) .. controls (258.35,190.26) and (228.9,187.09) .. (202.69,207.07) ;
\draw [shift={(201.5,208)}, rotate = 321.6] [fill={rgb, 255:red, 0; green, 0; blue, 0 }  ][line width=0.75]  [draw opacity=0] (8.93,-4.29) -- (0,0) -- (8.93,4.29) -- cycle    ;

\draw    (361.5,207) .. controls (338.35,190.26) and (308.9,187.09) .. (282.69,207.07) ;
\draw [shift={(281.5,208)}, rotate = 321.6] [fill={rgb, 255:red, 0; green, 0; blue, 0 }  ][line width=0.75]  [draw opacity=0] (8.93,-4.29) -- (0,0) -- (8.93,4.29) -- cycle    ;

\draw    (551.5,207) .. controls (528.35,190.26) and (498.9,187.09) .. (472.69,207.07) ;
\draw [shift={(471.5,208)}, rotate = 321.6] [fill={rgb, 255:red, 0; green, 0; blue, 0 }  ][line width=0.75]  [draw opacity=0] (8.93,-4.29) -- (0,0) -- (8.93,4.29) -- cycle    ;

\draw  [fill={rgb, 255:red, 248; green, 231; blue, 28 }  ,fill opacity=1 ] (25,96) .. controls (25,86.06) and (32.84,78) .. (42.5,78) .. controls (52.16,78) and (60,86.06) .. (60,96) .. controls (60,105.94) and (52.16,114) .. (42.5,114) .. controls (32.84,114) and (25,105.94) .. (25,96) -- cycle ;
\draw    (61,96) -- (102,96) ;
\draw [shift={(104,96)}, rotate = 180] [fill={rgb, 255:red, 0; green, 0; blue, 0 }  ][line width=0.75]  [draw opacity=0] (8.93,-4.29) -- (0,0) -- (8.93,4.29) -- cycle    ;

\draw  [fill={rgb, 255:red, 184; green, 233; blue, 134 }  ,fill opacity=1 ] (25,226) .. controls (25,216.06) and (32.84,208) .. (42.5,208) .. controls (52.16,208) and (60,216.06) .. (60,226) .. controls (60,235.94) and (52.16,244) .. (42.5,244) .. controls (32.84,244) and (25,235.94) .. (25,226) -- cycle ;
\draw    (63,226) -- (104,226) ;

\draw [shift={(61,226)}, rotate = 0] [fill={rgb, 255:red, 0; green, 0; blue, 0 }  ][line width=0.75]  [draw opacity=0] (8.93,-4.29) -- (0,0) -- (8.93,4.29) -- cycle    ;
\draw  [dash pattern={on 0.84pt off 2.51pt}] (88.5,54) -- (581,54) -- (581,137) -- (88.5,137) -- cycle ;
\draw  [dash pattern={on 0.84pt off 2.51pt}] (88.5,184) -- (581,184) -- (581,267) -- (88.5,267) -- cycle ;

\draw (124.5,95) node   {$0$};
\draw (204.5,95) node   {$x$};
\draw (284.5,95) node   {$2x$};
\draw (419,85) node [scale=2.488]  {$\dotsc $};
\draw (554.5,96) node   {$\sqrt{m'}x$};
\draw (325.5,15) node   {$\sqrt{m'}+1$};
\draw (364.5,95) node   {$3x$};
\draw (474.5,95) node [scale=0.5]  {$( \sqrt{m'}-1) x$};
\draw (124.5,225) node   {$0$};
\draw (204.5,225) node   {$x$};
\draw (284.5,225) node   {$2x$};
\draw (419,215) node [scale=2.488]  {$\dotsc $};
\draw (554.5,226) node   {$\sqrt{m'}x$};
\draw (364.5,225) node   {$3x$};
\draw (474.5,225) node [scale=0.5]  {$( \sqrt{m'}-1) x$};
\draw (43.5,95) node   {$0$};
\draw (43.5,225) node   {$0$};
\draw (44,68) node  [align=left] {initial};
\draw (44,195) node  [align=left] {final};
\draw (621,93) node  [align=left] {upper path};
\draw (622,226) node  [align=left] {lower path};

\end{tikzpicture}
	\caption{The search graph is depicted for a fixed temperature $\tau$. $x = \tau \ln 0.5$ is chosen in a way that $e^{-x/\tau} = 1/2$ holds.}\label{fig:lowera}
\end{figure}

%% file: tables/results.tex
\definecolor{LightCyan}{rgb}{0.88,1,1}

\begin{table}[!htbp]
	\centering
	\begin{tabular}{|c|c|c|}
		\hline
		\rowcolor{LightCyan}\multicolumn{3}{|c|}{Sample complexity}\\
		\hline
		 \multicolumn{1}{|l}{Upper bound:} & 	\multicolumn{2}{|c|}{$\tilde O(\sqrt{m})$}\\
		\hline	
		  \multicolumn{1}{|l}{Lower bound:} & 	\multicolumn{2}{|c|}{$\tilde \Omega(\sqrt{m})$}\\
		 \multicolumn{1}{|l}{(for our discretization approach)} &\multicolumn{2}{|c|}{}\\
		\hline	
		\multicolumn{1}{|l}{Lower bound:}  & 	\multicolumn{2}{|c|}{$\tilde \Omega(m^{1/3})$}\\
		\multicolumn{1}{|l}{ (for any learning algorithm)} &\multicolumn{2}{|c|}{}\\
		\hline
		\rowcolor{LightCyan}\multicolumn{3}{|c|}{Simulation complexity}\\
	\hline
\textsf{identical paths} & \textsf{separate paths}  &\textsf{separate paths} + \allsatisfied\\
\hline
exact solution & exact solution &$(O(\log n|T|),0)$ approximation\\
in time  & in time & in time \\
$\mathsf{poly}(m,n,|T|)$ & $\mathsf{poly}(m,n,|T|^n)$ &$\mathsf{poly}(m,n,|T|)$\\
\hline
	\end{tabular}
\caption{In the computational results, $T$ is the set of the discretized temperatures and $n$ is the number of sampled used in the learning algorithm. We show in Section \ref{sec:discretization} that $|T| = \tilde O(\sqrt{m})$ and $n = \tilde O(\sqrt{m})$ is almost without loss of generality but we treat them as separate parameters for the sake of generality.}
\end{table}\label{table:resultso}

%% file: src/discretization.tex
 \section{Discretization and Sample Complexity}\label{sec:discretization}

In this section, we give an analysis for the sample complexity of the problem. Recall that, for any problem instance $\ii$ and any sequence of  $m$ temperatures $\para = \langle t_{1}, t_{2}, \ldots, t_{m}\rangle$, we define $ \score(\ii,\para)$ to be the probability of finding an acceptable solution of $\ii$ using temperatures in $\para$. We say $\para$ is $\varepsilon$-approximately optimal, if $\E_{\ii \sim \distribution} \score(\ii,\para) \geq \sup_{\para'} \E_{\ii \sim \distribution} \score(\ii,\para') - \varepsilon$. That is, no other cooling schedule of the same length can achieve a significantly higher success rate. Our goal is to prove that learning an $\varepsilon$-approximately optimal cooling schedule only requires a polynomial number of i.i.d. samples from $\distribution$.

 One of the difficulties in finding near-optimal cooling schemes is that there are infinitely many options available. We show that by discretizing the temperatures into $\tilde O(m/\varepsilon)$ different values, we only lose an additive error of $\varepsilon$ in the success rate when running the algorithm on any instance of the problem.  Note that, we are not making any assumptions yet: we only rely on the fact that the algorithm is evaluated based on the success rate. Discretizing the temperature makes designing efficient algorithms possible too as we will show in Section \ref{sec:bestsequence}. Our main result is an upper bound of $\tilde O(\sqrt{m})$ for the sample complexity which is explained in details later in Section \ref{sec:improved-upper}. Here we start as a warm-up by giving an upper bound of $\tilde O(m)$. 
\begin{theorem}\label{thm:sample}
	The sample complexity of computing an $\varepsilon$-approximately optimal cooling schedule with length $m$ is bounded by $O \left(\varepsilon^{-2} \left(m \log(\frac{m \emax}{\varepsilon}) \right)\right)$.
\end{theorem}
\begin{proof}
Recall that we assume that the energies of the nodes are in set $\{0, 1, 2, \cdots, \emax\}$.
 The proof can be divided into the following steps:
	\begin{itemize}
		\item We start by showing that it is possible to discretize the temperatures to $T = \{ d_{1} , d_{2} , d_{3} , \ldots , d_{\tsize}\}$, such that for any sequence of $m$ temperatures $\para = \langle t_{1}, t_{2}, \ldots, t_{m}\rangle$, there exists a sequence of $m$ discrete temperatures $\para' = \langle t_{1}', t_{2}', \ldots, t_{m}'\rangle \in T^m$, such that
		\[ | \score(\ii,\para) - \score(\ii,\para')| \leq \varepsilon/3  \]
		for any instance $\ii$ of the problem. In other words, the discretized temperatures preserve the score approximately.
		\item Then, we show the sample complexity of learning an $\varepsilon/3$-approximately optimal temperature  $\para_{\opt}$ in $T^m$ is polynomial.	 This is achieved by standard concentration results in finite hypothesis space since $T^m$ has only a finite number of cooling schedules.
		\item Finally, we conclude that $\para_{\opt}$ is $\varepsilon$-approximately optimal in $\mathbb{R}_{\ge 0}^m$.
	\end{itemize}
	
	Define a parameter $\delta$ in a way that $ (1-\delta)^m = 1- \varepsilon/3$, that is, $\delta = \Theta(\frac{\varepsilon}{m})$. 
	We  first construct our discretized temperatures as  $T = T_1 \cup T_2 \cup \cdots \cup T_{\emax}$, where
	\begin{equation}\label{Defn:Tj}T_j = \{ \frac{j}{\ln(1/(i \delta))} \hspace{0.3cm}| \hspace{0.3cm}\forall \textit{ integer } 1 \leq i \leq \lceil 1/\delta \rceil \}.\end{equation} Roughly speaking, this discretization has the nice property that for any $1 \leq j \leq \emax$, set $\{e^{  j/t } \hspace{0.3cm}| \hspace{0.3cm}\forall t \in T_j\}$ evenly divides $[0,1]$. Therefore, for each temperature $t$, we can find a nearest neighbor $\tilde{t}$ in $T$ defined as $$\tilde{t} := \arg \min_{t \in T } \left| t-\tilde{t} \right|,$$ which implies 
	$\left| e^{ j/\tilde{t}  } - e^{ j/t } \right|  \leq \delta$ for any $0 \leq j \leq \emax$.
    Notice that the value of $\Delta(E)$ in our \textsf{SA} algorithm is always in range $\{0,1,\ldots,\emax\}$ and thus for $t$ and $\tilde{t}$, $e^{\Delta(E)/t}$ and $e^{\Delta(E)/\tilde{t}}$ are always within an additive range of $\delta$ regardless of the value of $\Delta(E)$.
	Our key observation is that for any sequence $\para = \langle t_1, t_2, \ldots, t_m\rangle \in \mathbb{R}_{\ge 0}^{m} $, there exists a sequence of $m$ temperatures $\para' = \langle t_{1}', t_{2}', \ldots, t_{m}'\rangle \in T^m$, such that running the simulated annealing algorithms with discrete temperature in $\para'$ keeps the trajectories the same as  $\para$ with probability at least $1-\delta$.  To this end, we define $t_i' = \tilde{t_i}$. Assuming the two runs share the same randomness, then these two runs are the same at each step with probability at least $1-\delta$. We only need to check the correctness for two cases, when a move is a downhill move or a uphill move. 
	
	For an uphill move, the correctness is  obvious since both runs accept the move with probability $1$. For a downhill move, the accepting probability are $ e^{ \Delta(E)/t_i  } $ and  $ e^{ \Delta(E)/t_i'} $, respectively. By choosing $t_i' = \tilde{t}_{i}$, the difference is at most	
	\begin{equation}
	|e^{  \Delta(E)/t_i  } - e^{  \Delta(E)/t_i'  }| \leq \delta.
	\end{equation} 	
	Therefore, we have proved that for each step, the  two runs are the same with probability $1- \delta$, hence they remain the same at all steps with probability at least $(1-\delta)^{m} = 1-\epsilon/3$. Assuming the score function is bounded in $[0,1]$, then the scores are different with at most $1$ when the two runs are different. Hence, the expectation of difference is upper bounded by	
	$$ | \score(\ii,\para) - \score(\ii,\para')| \leq 1 - (1-\delta)^{m}  = \frac{\varepsilon}{3} $$	
	Next, we will show that finding a near-optimal cooling schedule in $T^m$ requires polynomial sample complexity. The technique is based on standard Hoeffding and union bounds. We define $n$ as the upper bound on the number of samples and let $\ii_1, \ii_2, \cdots, \ii_n$ be $n$ problem instances sampled i.i.d. from $\distribution$ and $\empirical$ be a uniform distribution over $\{\ii_1, \ii_2, \cdots, \ii_n\}$.  For a given sequence of temperatures $\para  \in T^m$, by Hoeffding's Inequality we have 
	\footnote{Hoeffding's Inequality: Let $X_1, X_2, \cdots, X_n \sim_{i.i.d.} P$ and $X_i \in [0,1]$, then $ |\frac{1}{n}\sum_{i=1}^n X_i - \E[\frac{1}{n}\sum_{i=1}^n X_i]|\leq \varepsilon$ holds with probability at least $1- 2 e^{-2n\varepsilon^2}$ }
	\[\left| \E_{\ii \sim \distribution} \score(\ii,\para) -\E_{\ii \sim \empirical} \score(\ii,\para) \right| \leq \frac{\varepsilon}{3}   \]
	with probability at least $ 1 - e^{-\frac{2}{9} n\varepsilon^2 }$. Therefore, by union bound, this inequality holds for all $\para \in T^m$ with probability at least $1 - |T|^m e^{-\frac{2}{9}n\varepsilon^2 }$. Since we would like this event to happen with high probability, we wish to give a value to $n$ to make sure
	\begin{equation}\label{eq:1}
	1 - |T|^m e^{-\frac{2}{9}n\varepsilon^2} \geq 1 - m^{-10}
	\end{equation}
Define $\para_\opts = \arg\min_{\para \in T^m} \E_{\ii \sim \empirical}\score(\ii, \para)$ be the empirically best discretized cooling schedule, and $\para_\optd = \arg\min_{\para \in T^m} \E_{\ii \sim \distribution}\score(\ii, \para)$ be the population best discretized cooling schedule. Condition on the two events above, we have: 
	\begin{align*}
	 \E_{\ii \sim \distribution} \score(\ii,\para_{\opts}) - \sup_{\para \in \mathbb{R}_{\geq 0}^m} \E_{\ii \sim \distribution}  \score(\ii,\para)
	= & \quad \E_{\ii \sim \distribution} \score(\ii,\para_{\opts}) -  \E_{\ii \sim \empirical} \score(\ii,\para_{\opts}) \\[-10pt]
	& +  \E_{\ii \sim \empirical} \score(\ii,\para_{\opts}) - \E_{\ii \sim \empirical} \score(\ii,\para_{\optd})\\
	&  + \E_{\ii \sim \empirical} \score(\ii,\para_{\optd}) - \E_{\ii \sim \distribution}  \score(\ii,\para_{\optd})\\
	& + \E_{\ii \sim \distribution} \score(\ii,\para_{\optd}) - \sup_{\para \in \mathbb{R}_{\geq 0}^m} \E_{\ii \sim \distribution}  \score(\ii,\para) \\[-10pt]
	\geq & -\frac{\varepsilon}{3} + 0 - \frac{\varepsilon}{3}  - \frac{\varepsilon}{3}  \geq -\varepsilon
	\end{align*}
	Hence, we have proved
	\begin{equation}
		\E_{\ii \sim \distribution} \score(\ii,\para_{\opts}) \geq  \sup_{\para \in \mathbb{R}_{\geq 0}^m} \E_{\ii \sim \distribution}  \score(\ii,\para) - \varepsilon
	\end{equation}
	In other words,  $\opts$ is $\varepsilon$-approximately optimal.
	
	\textbf{Sample complexity}: we need to set $n$ in a way that meets Inequality \eqref{eq:1}, i.e. $|T|^m e^{-cn\varepsilon^2} \leq m^{-10}$. Therefore, we have
	$$n = \Theta\left(\varepsilon^{-2} m \log(|T|) \right) =  \Theta\left(\varepsilon^{-2} m \log(\frac{m \emax}{\varepsilon}) \right).$$
\end{proof}

The dependence of the above bound on $\emax$ is logarithmic which is loose when $\emax$ can obtain exponentially large values. We show that this can be further improved. More precisely, we show this by a more careful construction of $T$, such that $|T| = \Theta(\frac{m\log(\emax)}{\delta}) = \Theta(\frac{m^2\log(\emax)}{\varepsilon})$. Therefore, we can improve the upper bound on the sample complexity to  	
	$$n = O\left(\varepsilon^{-2} m \log(|T|) \right) =  O\left(\varepsilon^{-2} m \log\left(\frac{m\log(\emax)}{\varepsilon}  \right) \right).$$
We construct the discretized temperatures as $T = \bigcup_{j \in J}T_j $, where $J = \{ 1,(1+\delta), (1+\delta)^2, (1+\delta)^3, \cdots, \emax \}$ and $T_j$ defined as \eqref{Defn:Tj}. 
In order to improve the sample complexity, it suffices to show that for each temperature $t$ there is a  $\tilde{t}$ in $T$ such that 
$$\left| e^{ \Delta(E)/\tilde{t}  } - e^{ \Delta(E)/t } \right|  \leq O(\delta)$$  holds for all and $\Delta(E) \in [0, \emax]$.

By the definition of $J$, there always exists a $j^* \in J$, such that $\Delta(E) \leq j^* \leq (1+\delta) \Delta(E)$. Recall that our discretization has the nice property that for any $j \in J$, set $\{e^{  j/t } \hspace{0.3cm}| \hspace{0.3cm}\forall t \in T_j\}$ evenly divides $[0,1]$. Therefore, there exists a $\tilde{t} \in T_{j^*}$, such that 
$ |e^{  j^*/t } - e^{  j^*/\tilde{t} }| \leq \delta.$


Using Observation \ref{obs:tavan}, now we can bound the difference $ \left| e^{ \Delta(E)/\tilde{t}  } - e^{ \Delta(E)/t } \right| $:

\begin{align*}
\left| e^{ \Delta(E)/\tilde{t}  } - e^{ \Delta(E)/t } \right|	\leq &  \left| e^{ \Delta(E)/\tilde{t}  } - e^{ j^*/\tilde{t} } \right| + \left| e^{ j^*/\tilde{t}  } - e^{ j^*/t } \right| + \left| e^{ j^*/t  } - e^{ \Delta(E)/t } \right| \\
	\leq & O(\frac{j^*}{\Delta(E)} -1) + \delta + O(\frac{j^*}{\Delta(E)} -1)  \\
	= &  O(\delta)
\end{align*}
and the rest of the proof remains the same. 

\begin{corollary}[of Theorem \ref{thm:sample}]\label{color:sample}
	The sample complexity of computing an $\varepsilon$-approximately optimal cooling schedule with length $m$ is bounded by $O \left(\varepsilon^{-2} \left(m \log(\frac{m \log \emax}{\varepsilon}) \right)\right)$.
\end{corollary}

%% file: src/lowerbound.tex
\newpage
\section{Lower Bound}\label{sec:lower}
This section is dedicated to proving a lower bound for the sample complexity of any algorithm. Similar to our upper bound, our lower bound is also very general and without any assumptions. We show that any algorithm that approximates the optimal schedule within a small additive error requires at least $\tilde \Omega(m^{1/3})$ samples from the distribution. 

The overall idea of the proof is summarized in the following. We construct $l = |L| = \tilde \Omega(m^{1/3})$ different search graphs $L = \{s_1, s_2, \ldots, s_l\}$. Our construction has a nice property that each search graph requires a certain sequence of temperatures to be present in the cooling schedule in order to find a desirable solution after at most $m$ steps. We refer to such sequences as \textit{keys}. For each search graph, having its key in the cooling schedule guarantees that the search graph is traversed successfully with high probability when we use that cooling schedule. However, the length of each key is smaller than $m$ which allows us to bring multiple keys in an almost optimal solution. The keys are designed in a way that they do not share any elements in common. That is, a temperature used for a key specific to a search graph offers little benefit to the other search graphs. Our distribution $\distribution$ is a uniform distribution over a subset $L_{\distribution} \subseteq L$ which contains $\tilde \Theta(l)$ (but much smaller than $l$) search graphs from $L$. The crux of the argument is that by knowing $L_{\distribution}$, one can construct a sequence of size $m$ which includes all the keys of the search graphs in $L_{\distribution}$ that achieves a score close to 1 on average. However, $L_{\distribution}$ is unknown to the learner and if we draw fewer than $\tilde \Omega(l)$ samples, there is no hope to get any score more than $0.1$. Therefore, any learning scheme needs at least $ \tilde \Omega(l) = \tilde \Omega(m^{1/3})$ samples from the distribution to report an approximate solution.

Let $c = 100$ be a large constant. Recall that $m$ is the length of the optimal solution. We define a parameter $m' = \tilde \Theta(m^{2/3})$ which determines both the width of each gadget and the size of the key for each gadget. More precisely, we set the width of each gadget to $\sqrt{m'}$ and the size of the key for each gadget to $2cm'$. Let us first explain how each gadget is constructed and then show how the gadgets can be used to prove a lower bound on the sample complexity.

\input{figs/lower2}

Each gadget is made for a specific temperature. We fix the temperature to be $\tau$ and construct the corresponding gadget, namely $\gad(\tau)$ in the following way: As shown in Figure \ref{fig:lower}, our gadget is constructed of two identical paths. In the upper path,
the first node  has an energy of $0$ and has an outgoing edge to the second vertex. For the next $\sqrt{m'}-2$ vertices, vertex $i+1$ has an energy of $ix$ and three outgoing edges: 1) two edges to vertex $i$ and one edge to vertex $i+2$. Finally, the last node has an energy of $\sqrt{m'}x$ and has two outgoing edges to vertex $\sqrt{m'}$. $x = (\ln 1/2)\tau$ is set in a way that when the temperature is equal to $\tau$ the probability of accepting a downhill move is exactly equal to $1/2$. 

The lower path is constructed exactly the same way as the upper path. To connect the two paths together, we put an edge from the last node of the upper path to the last node of the lower path. Finally we add two dummy nodes to the search graph. The first dummy node has a single outgoing edge to the first vertex of the upper path and the second dummy node has a single incoming edge from the first node of the lower path. The goal of the gadget is to start from the first dummy node and reach the second dummy node. We call the first and the second dummy nodes the \textit{initial} and \textit{final} nodes respectively.

We define the key $\key(\tau)$ to be a sequence of size $2cm'$ only containing temperature $\tau$. As shown in Lemma \ref{lemma:ss1}, starting from an arbitrary node of $\gad(\tau)$ and running the \textsf{SA} algorithm on cooling schedule $\key(\tau)$ our algorithm ends at the final node with probability at least $0.9$.

Before bringing the proof, we state an observation for which we provide a proof in the appendix.
\begin{observation}\label{obs:main}
	Let $c = 100$, $x_0 = 0$ and $x_1,x_2,\ldots,x_k$ be $k$ variables constructed in the following way:
	\begin{equation*}
	\begin{cases}
		x_{i-1}-1 & \textsf{with probability }p_b, \\
		x_{i-1} & \textsf{with probability }p_s, \\
		x_{i-1}+1 & \textsf{with probability }p_f.
	\end{cases}
	\end{equation*}
	Then we have:
	\begin{enumerate}
		\item For $p_b = p_s = p_f = 1/3$ we have $\max \{x_i\} \geq \sqrt{k/c}+2$ with probability at least $0.95$.
		\item For $p_b = p_s = p_f = 1/3$ we have $\max \{x_i\} < \sqrt{ck} \log k$ with probability at least $1-1/k^2$.
		\item For any $k' \leq k$, $p_b \geq 1/3+\frac{c \log k'}{\sqrt{k'}}$, $p_f \leq 1/3-\frac{c\log k'}{\sqrt{k'}}$, $p_s = 1-p_b-p_f$   we have $\max \{x_i\} < \sqrt{k'}/2$ with probability at least $1-1/k'^2$.
	\end{enumerate}
\end{observation}

\begin{lemma}\label{lemma:ss1}
	An \textsf{SA} algorithm that starts from any node of $\gad(\tau)$ and runs on cooling schedule $\key(\tau)$ ends at the final node with probability at least $0.9$.
\end{lemma}
\begin{proof}
	We prove the lemma for an \textsf{SA} algorithm that starts from the initial node. Indeed this implies the lemma for any other starting node since in order to reach the final node, one needs to traverse all nodes of the search graph starting from the initial node.
	
	To this end, we  show that after $cm'$ steps our \textsf{SA} algorithm reaches the last node of the lower-path with probability at least $0.95$. With a similar analysis, one can show that starting from the last node of the lower-path, after $cm'$ steps our algorithm reaches the final node with probability at least $0.95$ after $cm'$ steps. Then, by applying the union bound, we imply that after $2cm'$ steps, our algorithm reaches the final node with probability at least $0.9$.
	
	From here on, our aim is to prove that starting from the initial node, our algorithm reaches the last node of the lower-path with probability at least $0.95$ after $cm'$ steps. Notice that since the temperature is always equal to $\tau$, in every step, our node in the search graph gets closer to the destination with probability at least $1/3$ and get farther from the destination with probability at most $1/3$. Due to Observation \ref{obs:main} (item (i)) after $cm'$ steps, with probability at least $0.95$ at some point the number of times we go forward is at least $\sqrt{m'}+2$ more than the number of times we go backward which means we reach the last node of the lower-path. This implies that with probability at least $0.95$ our algorithm reaches the last node of the lower-path after $cm'$ steps. A similar analysis proves that the next $cm'$ steps take us to the final node with probability at least $0.9$ which implies that $2cm'$ steps suffices to reach the final node with probability at least $0.9$.
\end{proof}

We also show that any cooling schedule needs a certain amount of temperatures close to $\tau$ to reach the final node with a considerable probability. 

\begin{lemma}\label{lemma:ss2}
	Let $\para$ be a cooling schedule of length $m$ containing no more than $\frac{m'}{4c \log^2 m'}$ temperatures in range $[\tau\frac{\sqrt{m'}-c^2 \log m'}{\sqrt{m'}},\tau\frac{\sqrt{m'}+c^2 \log m'}{\sqrt{m'}}]$. If an \textsf{SA} algorithm starts from the initial node and runs with cooling schedule $\para$, the probability that it reaches the final node is at most $0.1$.
\end{lemma}
\begin{proof}
	The intuition behind the proof is the following: For the upper-path, we would like to go to the right and thus a low temperature is desirable. For the lower-path however, since we would like to go to the left, we would like the temperature to be as high as possible. The key point is that in the cooling schedule, the temperatures are decreasing, thus either all the temperatures we use for traversing the upper-path are at least $\tau$ or all of the temperatures we use for traversing the lower-path are bounded by $\tau$. Any one of the two events makes it unlikely to get a high score.
	
	We assume w.l.o.g that we would like to traverse the upper-path with temperatures higher than $\tau$. Notice however that except for $\frac{m'}{4c \log^2 m'}$ temperatures, all the rest are more than $\tau$ by a multiplicative factor of $\frac{\sqrt{m'} + c^2 \log m'}{\sqrt{m'}}$. Since we strictly favor lower temperatures, the most desirable cooling schedule in this case is a sequence of $m-\frac{m'}{4c \log^2 m'}$ temperatures $\tau\frac{\sqrt{m'}+c^2 \log m'}{\sqrt{m'}}$ followed by $\frac{m'}{4c \log^2 m'}$ temperatures $\tau$. We show that it is still very unlikely to traverse the upper-path using this sequence. 
	
	To keep the analysis simple, we avoid the edge cases and assume that the goal is to start from the second vertex and never go back to the first vertex. This way, the probability of going forward or going backward only depends on the temperature and does not depend on the current vertex. If the temperature is equal to $\tau$ then with probability $p_f = 1/3$ we go forward and with probability $p_b=1/3$ we go backward. If the temperature is $\tau\frac{\sqrt{m'}+c^2 \log m'}{\sqrt{m'}}$ we go backward and with probability at least $p_b \geq 1/3+\frac{c\log m'}{\sqrt{m'}}$ we go forward with probability at most $p_f \leq 1/3-\frac{c\log m'}{\sqrt{m'}}$. Due to Observation \ref{obs:main}, if we proceed  $\frac{m'}{4c \log^2 m'}$ steps with temperature $\tau$ or $m$ steps with temperature $\tau\frac{\sqrt{m'}+c^2 \log m'}{\sqrt{m'}}$, our position does not improve by more that $\sqrt{m'}/2$ with probability at least $1-\tilde O(1)/m'^2$. Thus, in total the amount of improvement is bounded by $\sqrt{m'}$ with probability at least $1-\tilde O(1)/m'^2$.
	
	The above analysis fails when we bring in to the setting the first node of the upper-path since the probability of going to the right at this node is more than other nodes. However, we make the following argument: in order to traverse the upper-path, at some point we reach the second node of the upper-path and never go back. Let us say this happens at step $i$. Thus, from step $i$ on, we never go backwards and therefore all the probabilities are only a function of the temperature (and not the current node). The downside however, is that there are $m$ different possible choices for $i$ which multiplies the bad event probability by $m$. However, since we show in the above that increasing the position by an additive term of $\sqrt{m'}$ is not possible with probability $1-\tilde O(1)/m'^2$, we can imply by union bound that starting from any position $i$, increasing the position by an additive term $\sqrt{m'}$ is not possible with probability at least $1-\tilde O(m)/m'^2 << 0.1$ (for a large enough choice of $m$) which completes the proof.
\end{proof}

Now we are ready to prove the lower bound using Lemmas \ref{lemma:ss1} and \ref{lemma:ss2}.

\begin{theorem}\label{theorem:lower}
	Even if $\emax  = 2^{\tilde \Theta(1)}$, any learning algorithm requires at least $\tilde \Omega(m^{1/3})$ samples from the distribution in order to obtain an additive error less than $0.5$.
\end{theorem}
\begin{proof}
	As mentioned earlier, we have $l, |L_{\distribution}| = \tilde \theta(m^{1/3})$ and $m' = \tilde \Theta(m^{2/3})$.
	To be more precise,  we set $l = 40c m^{1/3} \log m$, $m' = m^{2/3}\log m/2c$ and $|L_{\distribution}| = m^{1/3}/\log m$.
	
	Assume for now that we have $l$ different temperatures $1 \leq \tau_1  < \tau_2 < \ldots < \tau_l$ such that their multiplicative distance is at least $\frac{\sqrt{m'}+10c^2 \log m'}{\sqrt{m'}}$.
	
	As outlined earlier, $L_{\distribution}$ is a uniform distribution over $m^{1/3}/\log m$ search graphs corresponding to temperatures $\tau_1, \tau_2, \ldots, \tau_l$. Each combination has equal probability of forming $L_{\distribution}$. Distribution $\distribution$ is a uniform distribution over the search graphs corresponding to the elements of  $L_{\distribution}$. The optimal solution consists of the keys for all the search graphs corresponding to the temperatures of $L_{\distribution}$. Since the size of the key for each search graph is $2cm' = m^{2/3} \log m$ and  $|L_{\distribution}| =m^{1/3}/\log m $, this makes a cooling schedule of size $m$. Lemma \ref{lemma:ss1} implies that the score of such a cooling schedule is at least $0.9$ on average.
	
	On the other hand, after drawing fewer than $m^{1/3}/(100 \log m)$ samples, we can get a score of $1$ for at most a $0.01$ fraction of the search graphs of $L_{\distribution}$ but the average score for the rest of the instances would be smaller than $0.2$ by Lemma \ref{lemma:ss2} (Notice that the gap between the temperatures is large enough). Thus, $\Omega(m^{1/3}/\log m)$ samples are necessary to obtain an additive error smaller than $0.5$.
	
	To construct the temperatures we do the following: We set $x_1 = 1$ and for $1 < i \leq l$ we set $x_i = \lceil x_{i-1} \frac{\sqrt{m'}+10c^2 \log m'}{\sqrt{m'}}+1 \rceil$. Finally we set $\tau_i = x_i / \ln 2$ to obtain $e^{-x_i / \tau_i} = 0.5$. To make sure all the energies are non-zero, we add $1$ to the energy of all nodes in all gadgets. 
\end{proof}

%% file: figs/lower2.tex
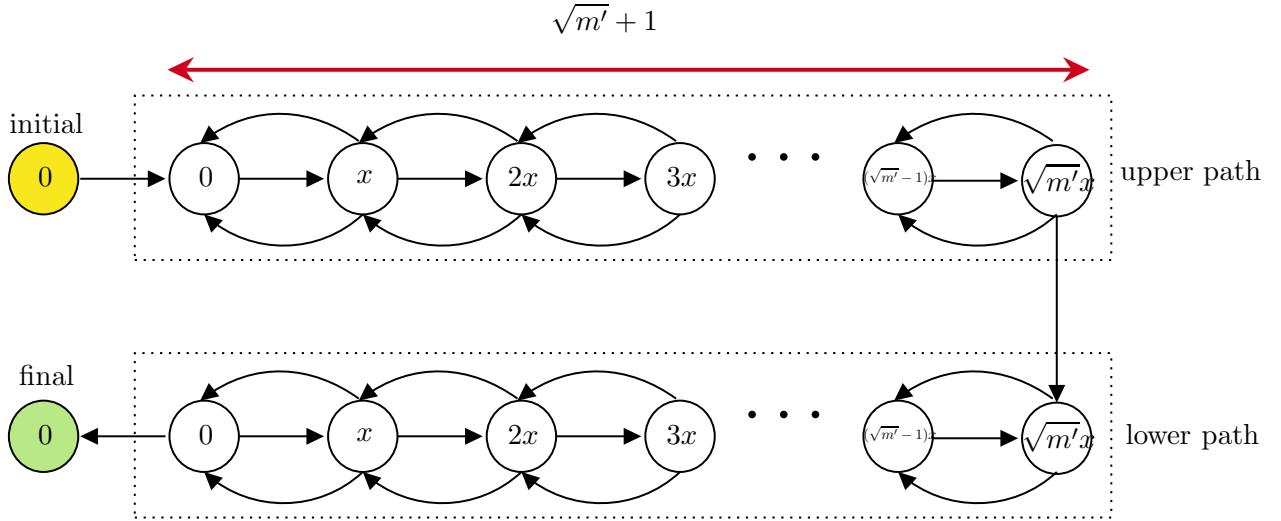
\begin{figure}[H]\centering

\tikzset{every picture/.style={line width=0.75pt}} 

\begin{tikzpicture}[x=0.75pt,y=0.75pt,yscale=-1,xscale=1]

\draw   (106,96) .. controls (106,86.06) and (113.84,78) .. (123.5,78) .. controls (133.16,78) and (141,86.06) .. (141,96) .. controls (141,105.94) and (133.16,114) .. (123.5,114) .. controls (113.84,114) and (106,105.94) .. (106,96) -- cycle ;
\draw    (141,96) -- (182,96) ;
\draw [shift={(184,96)}, rotate = 180] [fill={rgb, 255:red, 0; green, 0; blue, 0 }  ][line width=0.75]  [draw opacity=0] (8.93,-4.29) -- (0,0) -- (8.93,4.29) -- cycle    ;

\draw   (186,96) .. controls (186,86.06) and (193.84,78) .. (203.5,78) .. controls (213.16,78) and (221,86.06) .. (221,96) .. controls (221,105.94) and (213.16,114) .. (203.5,114) .. controls (193.84,114) and (186,105.94) .. (186,96) -- cycle ;
\draw    (221,96) -- (262,96) ;
\draw [shift={(264,96)}, rotate = 180] [fill={rgb, 255:red, 0; green, 0; blue, 0 }  ][line width=0.75]  [draw opacity=0] (8.93,-4.29) -- (0,0) -- (8.93,4.29) -- cycle    ;

\draw   (266,96) .. controls (266,86.06) and (273.84,78) .. (283.5,78) .. controls (293.16,78) and (301,86.06) .. (301,96) .. controls (301,105.94) and (293.16,114) .. (283.5,114) .. controls (273.84,114) and (266,105.94) .. (266,96) -- cycle ;
\draw    (301,96) -- (342,96) ;
\draw [shift={(344,96)}, rotate = 180] [fill={rgb, 255:red, 0; green, 0; blue, 0 }  ][line width=0.75]  [draw opacity=0] (8.93,-4.29) -- (0,0) -- (8.93,4.29) -- cycle    ;

\draw    (491,97) -- (532,97) ;
\draw [shift={(534,97)}, rotate = 180] [fill={rgb, 255:red, 0; green, 0; blue, 0 }  ][line width=0.75]  [draw opacity=0] (8.93,-4.29) -- (0,0) -- (8.93,4.29) -- cycle    ;

\draw   (536,97) .. controls (536,87.06) and (543.84,79) .. (553.5,79) .. controls (563.16,79) and (571,87.06) .. (571,97) .. controls (571,106.94) and (563.16,115) .. (553.5,115) .. controls (543.84,115) and (536,106.94) .. (536,97) -- cycle ;
\draw    (203.5,114) .. controls (181.45,134.58) and (147.39,134.99) .. (124.86,115.23) ;
\draw [shift={(123.5,114)}, rotate = 403.03] [fill={rgb, 255:red, 0; green, 0; blue, 0 }  ][line width=0.75]  [draw opacity=0] (8.93,-4.29) -- (0,0) -- (8.93,4.29) -- cycle    ;

\draw [color={rgb, 255:red, 208; green, 2; blue, 27 }  ,draw opacity=1 ][line width=1.5]    (108,41) -- (567,41) ;
\draw [shift={(570,41)}, rotate = 180] [fill={rgb, 255:red, 208; green, 2; blue, 27 }  ,fill opacity=1 ][line width=1.5]  [draw opacity=0] (13.4,-6.43) -- (0,0) -- (13.4,6.44) -- (8.9,0) -- cycle    ;
\draw [shift={(105,41)}, rotate = 0] [fill={rgb, 255:red, 208; green, 2; blue, 27 }  ,fill opacity=1 ][line width=1.5]  [draw opacity=0] (13.4,-6.43) -- (0,0) -- (13.4,6.44) -- (8.9,0) -- cycle    ;
\draw    (201.5,77) .. controls (178.35,60.25) and (148.9,57.09) .. (122.69,77.07) ;
\draw [shift={(121.5,78)}, rotate = 321.6] [fill={rgb, 255:red, 0; green, 0; blue, 0 }  ][line width=0.75]  [draw opacity=0] (8.93,-4.29) -- (0,0) -- (8.93,4.29) -- cycle    ;

\draw    (553.5,115) -- (553.5,207) ;
\draw [shift={(553.5,209)}, rotate = 270] [fill={rgb, 255:red, 0; green, 0; blue, 0 }  ][line width=0.75]  [draw opacity=0] (8.93,-4.29) -- (0,0) -- (8.93,4.29) -- cycle    ;

\draw   (346,96) .. controls (346,86.06) and (353.84,78) .. (363.5,78) .. controls (373.16,78) and (381,86.06) .. (381,96) .. controls (381,105.94) and (373.16,114) .. (363.5,114) .. controls (353.84,114) and (346,105.94) .. (346,96) -- cycle ;
\draw   (456,96) .. controls (456,86.06) and (463.84,78) .. (473.5,78) .. controls (483.16,78) and (491,86.06) .. (491,96) .. controls (491,105.94) and (483.16,114) .. (473.5,114) .. controls (463.84,114) and (456,105.94) .. (456,96) -- cycle ;
\draw    (283.5,114) .. controls (261.45,134.58) and (227.39,134.99) .. (204.86,115.23) ;
\draw [shift={(203.5,114)}, rotate = 403.03] [fill={rgb, 255:red, 0; green, 0; blue, 0 }  ][line width=0.75]  [draw opacity=0] (8.93,-4.29) -- (0,0) -- (8.93,4.29) -- cycle    ;

\draw    (363.5,114) .. controls (341.45,134.58) and (307.4,134.99) .. (284.86,115.23) ;
\draw [shift={(283.5,114)}, rotate = 403.03] [fill={rgb, 255:red, 0; green, 0; blue, 0 }  ][line width=0.75]  [draw opacity=0] (8.93,-4.29) -- (0,0) -- (8.93,4.29) -- cycle    ;

\draw    (553.5,114) .. controls (531.45,134.58) and (497.4,134.99) .. (474.86,115.23) ;
\draw [shift={(473.5,114)}, rotate = 403.03] [fill={rgb, 255:red, 0; green, 0; blue, 0 }  ][line width=0.75]  [draw opacity=0] (8.93,-4.29) -- (0,0) -- (8.93,4.29) -- cycle    ;

\draw    (281.5,77) .. controls (258.35,60.25) and (228.9,57.09) .. (202.69,77.07) ;
\draw [shift={(201.5,78)}, rotate = 321.6] [fill={rgb, 255:red, 0; green, 0; blue, 0 }  ][line width=0.75]  [draw opacity=0] (8.93,-4.29) -- (0,0) -- (8.93,4.29) -- cycle    ;

\draw    (361.5,77) .. controls (338.35,60.25) and (308.9,57.09) .. (282.69,77.07) ;
\draw [shift={(281.5,78)}, rotate = 321.6] [fill={rgb, 255:red, 0; green, 0; blue, 0 }  ][line width=0.75]  [draw opacity=0] (8.93,-4.29) -- (0,0) -- (8.93,4.29) -- cycle    ;

\draw    (551.5,77) .. controls (528.35,60.25) and (498.9,57.09) .. (472.69,77.07) ;
\draw [shift={(471.5,78)}, rotate = 321.6] [fill={rgb, 255:red, 0; green, 0; blue, 0 }  ][line width=0.75]  [draw opacity=0] (8.93,-4.29) -- (0,0) -- (8.93,4.29) -- cycle    ;

\draw   (106,226) .. controls (106,216.06) and (113.84,208) .. (123.5,208) .. controls (133.16,208) and (141,216.06) .. (141,226) .. controls (141,235.94) and (133.16,244) .. (123.5,244) .. controls (113.84,244) and (106,235.94) .. (106,226) -- cycle ;
\draw    (141,226) -- (182,226) ;
\draw [shift={(184,226)}, rotate = 180] [fill={rgb, 255:red, 0; green, 0; blue, 0 }  ][line width=0.75]  [draw opacity=0] (8.93,-4.29) -- (0,0) -- (8.93,4.29) -- cycle    ;

\draw   (186,226) .. controls (186,216.06) and (193.84,208) .. (203.5,208) .. controls (213.16,208) and (221,216.06) .. (221,226) .. controls (221,235.94) and (213.16,244) .. (203.5,244) .. controls (193.84,244) and (186,235.94) .. (186,226) -- cycle ;
\draw    (221,226) -- (262,226) ;
\draw [shift={(264,226)}, rotate = 180] [fill={rgb, 255:red, 0; green, 0; blue, 0 }  ][line width=0.75]  [draw opacity=0] (8.93,-4.29) -- (0,0) -- (8.93,4.29) -- cycle    ;

\draw   (266,226) .. controls (266,216.06) and (273.84,208) .. (283.5,208) .. controls (293.16,208) and (301,216.06) .. (301,226) .. controls (301,235.94) and (293.16,244) .. (283.5,244) .. controls (273.84,244) and (266,235.94) .. (266,226) -- cycle ;
\draw    (301,226) -- (342,226) ;
\draw [shift={(344,226)}, rotate = 180] [fill={rgb, 255:red, 0; green, 0; blue, 0 }  ][line width=0.75]  [draw opacity=0] (8.93,-4.29) -- (0,0) -- (8.93,4.29) -- cycle    ;

\draw    (491,227) -- (532,227) ;
\draw [shift={(534,227)}, rotate = 180] [fill={rgb, 255:red, 0; green, 0; blue, 0 }  ][line width=0.75]  [draw opacity=0] (8.93,-4.29) -- (0,0) -- (8.93,4.29) -- cycle    ;

\draw   (536,227) .. controls (536,217.06) and (543.84,209) .. (553.5,209) .. controls (563.16,209) and (571,217.06) .. (571,227) .. controls (571,236.94) and (563.16,245) .. (553.5,245) .. controls (543.84,245) and (536,236.94) .. (536,227) -- cycle ;
\draw    (203.5,244) .. controls (181.45,264.58) and (147.39,264.99) .. (124.86,245.23) ;
\draw [shift={(123.5,244)}, rotate = 403.03] [fill={rgb, 255:red, 0; green, 0; blue, 0 }  ][line width=0.75]  [draw opacity=0] (8.93,-4.29) -- (0,0) -- (8.93,4.29) -- cycle    ;

\draw    (201.5,207) .. controls (178.35,190.26) and (148.9,187.09) .. (122.69,207.07) ;
\draw [shift={(121.5,208)}, rotate = 321.6] [fill={rgb, 255:red, 0; green, 0; blue, 0 }  ][line width=0.75]  [draw opacity=0] (8.93,-4.29) -- (0,0) -- (8.93,4.29) -- cycle    ;

\draw   (346,226) .. controls (346,216.06) and (353.84,208) .. (363.5,208) .. controls (373.16,208) and (381,216.06) .. (381,226) .. controls (381,235.94) and (373.16,244) .. (363.5,244) .. controls (353.84,244) and (346,235.94) .. (346,226) -- cycle ;
\draw   (456,226) .. controls (456,216.06) and (463.84,208) .. (473.5,208) .. controls (483.16,208) and (491,216.06) .. (491,226) .. controls (491,235.94) and (483.16,244) .. (473.5,244) .. controls (463.84,244) and (456,235.94) .. (456,226) -- cycle ;
\draw    (283.5,244) .. controls (261.45,264.58) and (227.39,264.99) .. (204.86,245.23) ;
\draw [shift={(203.5,244)}, rotate = 403.03] [fill={rgb, 255:red, 0; green, 0; blue, 0 }  ][line width=0.75]  [draw opacity=0] (8.93,-4.29) -- (0,0) -- (8.93,4.29) -- cycle    ;

\draw    (363.5,244) .. controls (341.45,264.58) and (307.4,264.99) .. (284.86,245.23) ;
\draw [shift={(283.5,244)}, rotate = 403.03] [fill={rgb, 255:red, 0; green, 0; blue, 0 }  ][line width=0.75]  [draw opacity=0] (8.93,-4.29) -- (0,0) -- (8.93,4.29) -- cycle    ;

\draw    (553.5,244) .. controls (531.45,264.58) and (497.4,264.99) .. (474.86,245.23) ;
\draw [shift={(473.5,244)}, rotate = 403.03] [fill={rgb, 255:red, 0; green, 0; blue, 0 }  ][line width=0.75]  [draw opacity=0] (8.93,-4.29) -- (0,0) -- (8.93,4.29) -- cycle    ;

\draw    (281.5,207) .. controls (258.35,190.26) and (228.9,187.09) .. (202.69,207.07) ;
\draw [shift={(201.5,208)}, rotate = 321.6] [fill={rgb, 255:red, 0; green, 0; blue, 0 }  ][line width=0.75]  [draw opacity=0] (8.93,-4.29) -- (0,0) -- (8.93,4.29) -- cycle    ;

\draw    (361.5,207) .. controls (338.35,190.26) and (308.9,187.09) .. (282.69,207.07) ;
\draw [shift={(281.5,208)}, rotate = 321.6] [fill={rgb, 255:red, 0; green, 0; blue, 0 }  ][line width=0.75]  [draw opacity=0] (8.93,-4.29) -- (0,0) -- (8.93,4.29) -- cycle    ;

\draw    (551.5,207) .. controls (528.35,190.26) and (498.9,187.09) .. (472.69,207.07) ;
\draw [shift={(471.5,208)}, rotate = 321.6] [fill={rgb, 255:red, 0; green, 0; blue, 0 }  ][line width=0.75]  [draw opacity=0] (8.93,-4.29) -- (0,0) -- (8.93,4.29) -- cycle    ;

\draw  [fill={rgb, 255:red, 248; green, 231; blue, 28 }  ,fill opacity=1 ] (25,96) .. controls (25,86.06) and (32.84,78) .. (42.5,78) .. controls (52.16,78) and (60,86.06) .. (60,96) .. controls (60,105.94) and (52.16,114) .. (42.5,114) .. controls (32.84,114) and (25,105.94) .. (25,96) -- cycle ;
\draw    (61,96) -- (102,96) ;
\draw [shift={(104,96)}, rotate = 180] [fill={rgb, 255:red, 0; green, 0; blue, 0 }  ][line width=0.75]  [draw opacity=0] (8.93,-4.29) -- (0,0) -- (8.93,4.29) -- cycle    ;

\draw  [fill={rgb, 255:red, 184; green, 233; blue, 134 }  ,fill opacity=1 ] (25,226) .. controls (25,216.06) and (32.84,208) .. (42.5,208) .. controls (52.16,208) and (60,216.06) .. (60,226) .. controls (60,235.94) and (52.16,244) .. (42.5,244) .. controls (32.84,244) and (25,235.94) .. (25,226) -- cycle ;
\draw    (63,226) -- (104,226) ;

\draw [shift={(61,226)}, rotate = 0] [fill={rgb, 255:red, 0; green, 0; blue, 0 }  ][line width=0.75]  [draw opacity=0] (8.93,-4.29) -- (0,0) -- (8.93,4.29) -- cycle    ;
\draw  [dash pattern={on 0.84pt off 2.51pt}] (88.5,54) -- (581,54) -- (581,137) -- (88.5,137) -- cycle ;
\draw  [dash pattern={on 0.84pt off 2.51pt}] (88.5,184) -- (581,184) -- (581,267) -- (88.5,267) -- cycle ;

\draw (124.5,95) node   {$0$};
\draw (204.5,95) node   {$x$};
\draw (284.5,95) node   {$2x$};
\draw (419,85) node [scale=2.488]  {$\dotsc $};
\draw (554.5,96) node   {$\sqrt{m'}x$};
\draw (325.5,15) node   {$\sqrt{m'}+1$};
\draw (364.5,95) node   {$3x$};
\draw (474.5,95) node [scale=0.5]  {$( \sqrt{m'}-1) x$};
\draw (124.5,225) node   {$0$};
\draw (204.5,225) node   {$x$};
\draw (284.5,225) node   {$2x$};
\draw (419,215) node [scale=2.488]  {$\dotsc $};
\draw (554.5,226) node   {$\sqrt{m'}x$};
\draw (364.5,225) node   {$3x$};
\draw (474.5,225) node [scale=0.5]  {$( \sqrt{m'}-1) x$};
\draw (43.5,95) node   {$0$};
\draw (43.5,225) node   {$0$};
\draw (44,68) node  [align=left] {initial};
\draw (44,195) node  [align=left] {final};
\draw (621,93) node  [align=left] {upper path};
\draw (622,226) node  [align=left] {lower path};

\end{tikzpicture}
	\caption{The search graph is depicted for a fixed temperature $\tau$. $x = \tau \ln 0.5$ is chosen in a way that $e^{-x/\tau} = 1/2$ holds.}\label{fig:lower}
\end{figure}

%% file: src/improved-upperbound.tex
\newpage
\section{Improved Upper Bound}\label{sec:improved-upper}
We show in this section that the bound of Theorem \ref{thm:sample} can be significantly improved. The proof is based on two observations: 1) first we show that the discretized set of temperatures can be made smaller while keeping the additive error small and 2) the proof can be modified to improve the sample complexity using the new discretized set. We first start by explaining the former.

Our discretization is very similar to that of Theorem \ref{thm:sample} except that in the construction of the temperatures we allow for a multiplicative error of $\tilde \Theta(m^{-1/2})$ instead of $\Theta(1/m)$. This implies that the multiplicative distance between consecutive elements of $T$ is bounded by $1+\tilde \Theta(m^{-1/2})$ (instead of $1+\Theta(1/m)$). This obviously leaves us with a smaller set of temperatures which later can be used to improve the sample complexity but the crucial part of the analysis is to show this smaller set suffices to bound the error by a small $\epsilon$. We prove that for any sequence of temperatures $\para = \langle t_1, t_2, \ldots, t_m\rangle$, there exists another sequence $\para' = \langle t'_1, t'_2, \ldots, t'_m\rangle$ such that $t'_i \in T$ for all $1 \leq i \leq m$ and that the scores of $\para$ and $\para'$ are very close for every search graph. Obviously we set $t'_i$ as the largest element of $T$ which is not greater than $t_i$. Therefore we have $1 \leq t_i / t'_i \leq 1+\tilde O(m^{-1/2})$.

Let us introduce a \textit{deviation function} $\function(x):[0,1] \rightarrow [0,0.25] = x - x^2$ which plays an important role in the proof of Lemma \ref{lemma:challenge}. The proof of this section is rather mathematical and unintuitive. For more intuition and as to why such a strange function is necessary for the proof we encourage the reader to review Section \ref{section:introduction}. Before proceeding to the proof of Lemma \ref{lemma:challenge}, we state some properties of function $\function$ as auxiliary observations as well as some mathematical inequalities which are used in the proof of the bound. We defer the proofs of these observations to appendix.

\begin{observation}\label{observation:ajib1}
	Let $x,y \in [0,1]$ be two real values and $0 \leq p \leq 1$ be a multiplicative factor. Then we have:
	$$p\function(x)+(1-p)\function(y) \leq \function(px + (1-p)y) - \min\{p,1-p\}(x-y)^2.$$
\end{observation}

Since $(x-y)^2$ is always non-negative therefore Observation \ref{observation:ajib1} implies that $p\function(x)+(1-p)\function(y) \leq \function(px + (1-p)y)$ always holds. By recursing on this inequality we can extend it to the case of more than two variables.
\begin{observation}[as a corollary of Observation \ref{observation:ajib1}]\label{observation:ajib2}
	Let $p_1, p_2, \ldots,p_k$ be non-negative probabilities whose total sum is equal to 1 and $x_1, x_2, \ldots,x_k$ be $k$ real values in range $[0,1]$ . Then we have:
	$$\sum p_i\function(x_i) \leq \function(\sum p_i x_i).$$
\end{observation}

Also, we show that for two real numbers $0 \leq x,y \leq 1$ we have $|x-y| \geq |f(x)-f(y)|$.
\begin{observation}\label{obs:sadeh}
	For any two real numbers $0 \leq x,y \leq 1$ we have $|x-y| \geq |f(x)-f(y)|$.
\end{observation}

\begin{observation}\label{obs:tavan}
	For any $0 \leq p \leq 1$ and any $0 \leq x $ we have $$p - p^{1+x} \leq x.$$
\end{observation}

We also present a slightly modified version of Observation \ref{obs:tavan} which provides a better bound for limited $p$.
\begin{observation}\label{obs:tavan2}
	For any $0 < p < 1$ and any $0 \leq x \leq 1$ we have $$p - p^{1+x} \leq \max\{\ln 1/p, 1\} \min\{p,1-p\}x.$$
\end{observation}

Now we are ready to prove Lemma \ref{lemma:challenge}.
\begin{lemma}\label{lemma:challenge}
	Let $\ii$ be an instance of the underlying problem and $\para = \langle t_1, t_2, \ldots, t_m\rangle$ and $\para' = \langle t'_1, t'_2, \ldots, t'_m\rangle$ be two cooling schedules such that $1 \leq t_i / t'_i \leq 1+\frac{\epsilon m^{-1/2}}{4 \log m}$ for some $\epsilon > 0$. Then we have $$\score(\ii, \para') \geq \score(\ii,\para) - \epsilon.$$
\end{lemma}
\begin{proof}
	
	 Our proof is based on induction. Define $\para^{+k}$ ($\para'^{+k}$) to be a cooling schedule starting from element $k+1$ of $\para$ ($\para'$) ($\para^{+0} = \para$ and $\para'^{+0} = \para'$). We denote the vertices of the search graph by $u_1, u_2, \ldots$ (their number may be exponentially large) and define $\score_{u_i}(\ii,\para^{+k})$ as the average score we obtain if we initiate the search on node $u_i$ and run the algorithm using cooling schedule $\para^{+k}$. When $k = m$, then $\para^{+k}$ is empty which means $\score_{u_i}(\ii,\para^{+k})$ is either equal to 0 or 1 depending on whether $u_i$ is an acceptable solution node in the search graph. A similar notation also holds for $\para'$. Our aim is to prove that for any $u_i$ and $k$ we have $\score_{u_i}(\ii,\para'^{+k}) \geq \score_{u_i}(\ii,\para^{+k}) - \epsilon$ which immediately implies $\score(\ii,\para') \geq \score(\ii,\para) - \epsilon$. However, to use induction, we strengthen the hypothesis. We show that 
	\begin{equation}\label{eq:gher}
	\score_{u_i}(\ii,\para'^{+k}) \geq \score_{u_i}(\ii,\para^{+k}) - \epsilon'\left[\function(\score_{u_i}(\ii,\para^{+k}))+\frac{m-k}{m} \right]
	\end{equation}
	where $\epsilon' = \epsilon/2$. Notice that since the value of $\function$ is always in range $[0,0.25]$, Inequality \eqref{eq:gher} is already stronger than what we wish to prove in the end. The base case is when $k = m$ which means the random walk has terminated and that $\score_{u_i}(\ii,\para'^{+k}) = \score_{u_i}(\ii,\para^{+k})$. Thus, for a fixed $k < m$, provided that Inequality \eqref{eq:gher} holds for any vertex $u_i$ and $k' = k+1$, we show Inequality \eqref{eq:gher} holds for any pair $(u_i,k)$.

	Recall that in every step of the \textsf{SA} algorithm, we first randomly draw an outgoing edge of the current node and then decide whether we traverse through that edge or not. Therefore $$\score_{u_i}(\ii,\para'^{+k}) = \mathbb{E}_{u_j \sim N(u_i)}[\score_{u_i,u_j}(\ii,\para'^{+k})]$$ where $N(u_i)$ denotes the set of neighbors of vertex $u_i$ and $\score_{u_i,u_j}(\ii,\para'^{+k})$ is the score of node $u_i$ for the event that the drawn edge is $(u_i, u_j)$.
	
	Let us first fix an edge $(u_i,u_j)$ and introduce an edge variant of Inequality \eqref{eq:gher}, namely Inequality \eqref{eq:abargher} for which we give a proof in the following.
	\begin{equation}\label{eq:abargher}
	\score_{u_i,u_j}(\ii,\para'^{+k}) \geq \score_{u_i,u_j}(\ii,\para^{+k}) - \epsilon'\left[\function(\score_{u_i,u_j}(\ii,\para^{+k}))+\frac{m-k}{m} \right].
	\end{equation}
	For simplicity of notation, let us define $a = \score_{u_i}(\ii,\para^{+(k+1)})$ and $b = \score_{u_j}(\ii,\para^{+(k+1)})$. Similarly, define $a' = \score_{u_i}(\ii,\para'^{+(k+1)})$ and $b' = \score_{u_j}(\ii,\para'^{+(k+1)})$. If the energy of node $u_j$ is more than the energy of node $u_i$ then the decision is deterministic regardless of the temperature and we have 
	\begin{equation*}
	\score_{u_i,u_j}(\ii,\para^{+k}) = b
	\end{equation*}
	and 
	\begin{equation*}
	\score_{u_i,u_j}(\ii,\para'^{+k}) = b'.
	\end{equation*}
	This implies that 
	\begin{align}
		\score_{u_i,u_j}(\ii,\para^{+k}) - \score_{u_i,u_j}(\ii,\para'^{+k}) &= b - b'\nonumber\\
		&\leq  \epsilon'\left[\function(b)+\frac{m-k-1}{m} \right]& \label{equation:saeed1}\\
		&=  \epsilon'\left[\function(\score_{u_i,u_j}(\ii,\para^{+k}))+\frac{m-k-1}{m} \right]\nonumber\\
		&\leq   \epsilon'\left[\function(\score_{u_i,u_j}(\ii,\para^{+k}))+\frac{m-k}{m} \right]\nonumber
	\end{align}
	where Inequality \eqref{equation:saeed1} follows from the induction hypothesis. This basically means that 
	$$\score_{u_i,u_j}(\ii,\para'^{+k}) \geq \score_{u_i,u_j}(\ii,\para^{+k}) - \epsilon'\left[\function(\score_{u_i,u_j}(\ii,\para^{+k}))+\frac{m-k}{m} \right]$$ which is desired. Thus, it only remains to prove Inequality \eqref{eq:abargher} for the cases that the energy decreases. This is the only case where $\para$ and $\para'$ behave differently. In this case, depending the temperatures $t_{k+1}$ and $t'_{k+1}$ our \textsf{SA} algorithm moves to node $u_j$ or stays at node $u_i$. Let $p$ be the probability of rejecting the downhill move to node $u_j$ when the temperature is equal to $t_{k+1}$ and $p'$ the same probability for the case that the temperature is $t'_{k+1}$. Recall that the acceptance probabilities are equal to $1-e^{-\Delta(E)/t_{k+1}}$ and $1-e^{-\Delta(E)/t'_{k+1}}$ (for $\para$ and $\para'$ respectively) where $\Delta(E)$ is the difference between the energies of nodes $u_i$ and $u_j$. Thus, $p = e^{-\Delta(E)/t_{k+1}}$ and $p' = e^{-\Delta(E)/t'_{k+1}}$ and since $1 \leq t_{k+1}/t'_{k+1} \leq 1+\frac{\epsilon m^{-1/2}}{4 \log m}$ then we have 
	$$p^{1+\frac{\epsilon m^{-1/2}}{4 \log m}} \leq p' \leq p.$$
	
	Note that $\score_{u_i,u_j}(\ii,\para^{+k})$ and $\score_{u_i,u_j}(\ii,\para'^{+k})$ can be formulated as 
	\begin{equation}
	\score_{u_i,u_j}(\ii,\para^{+k}) = p a + (1-p) b\label{eq:bikhod}
	\end{equation}
	 and 
	 \begin{equation}
	 \score_{u_i,u_j}(\ii,\para'^{+k}) = p' a' + (1-p') b'\label{eq:bikhod'}
	 \end{equation}
	  due to the acceptance probabilities. Thus, we have:
	\begingroup
	\allowdisplaybreaks
	\begin{align*}
	\score_{u_i,u_j}(\ii,\para'^{+k}) =\hspace{0.2cm}&p' a' + (1-p') b'\nonumber \\
	\geq\hspace{0.2cm}& p' \left[a - \epsilon' [\function(a) + \frac{m-k-1}{m}] \right] &\text{by induction hypothetis}\nonumber \\
	& +(1-p') \left[b - \epsilon' [\function(b) + \frac{m-k-1}{m}] \right] \nonumber\\
	=\hspace{0.2cm}& p' \left[a - \epsilon' \function(a)  \right] \nonumber \\
	& +(1-p') \left[b - \epsilon' \function(b)  \right] \nonumber\\
	& -\epsilon'\frac{m-k-1}{m} \nonumber\\
	=\hspace{0.2cm}& p \left[a - \epsilon' \function(a)  \right] \nonumber \\
	& +(1-p) \left[b - \epsilon' \function(b)  \right] \nonumber\\
	& -\epsilon'\frac{m-k-1}{m} \nonumber\\
	&-(p-p')([a - \epsilon' f(a)] - [b - \epsilon' f(b)])\\
	\geq\hspace{0.2cm}& p \left[a - \epsilon' \function(a)  \right] & p \geq p'\nonumber \\
	& +(1-p) \left[b - \epsilon' \function(b)  \right] \nonumber\\
	& -\epsilon'\frac{m-k-1}{m} \nonumber\\
	&-(p-p')(|a-b|+\epsilon'|f(a) - f(b)|)\\
	\geq\hspace{0.2cm}& p \left[a - \epsilon' \function(a)  \right] & p \geq p'\nonumber \\
	& +(1-p) \left[b - \epsilon' \function(b)  \right] &\text{and } \epsilon' \leq 1\nonumber\\
	& -\epsilon'\frac{m-k-1}{m} \nonumber\\
	&-(p-p')(|a-b|+|f(a) - f(b)|)\\
	\geq\hspace{0.2cm}& p \left[a - \epsilon' \function(a)  \right] & p \geq p' \nonumber \\
	& +(1-p) \left[b - \epsilon' \function(b)  \right] & \text{and } |a-b| \geq |f(a)-f(b)| \nonumber\\
	& -\epsilon'\frac{m-k-1}{m} & \text{(Observation \ref{obs:sadeh})} \nonumber\\
	&-2(p-p')|a-b|\\
	\geq\hspace{0.2cm}& [pa + (1-p)b] - \epsilon' f([pa + (1-p)b]) & \text{Observation \ref{observation:ajib1}}\nonumber \\
	&+ \epsilon'\min\{p,1-p\}(a - b)^2 \nonumber\\
	& -\epsilon'\frac{m-k-1}{m} \nonumber\\
	&-2(p-p')|a-b|\\
	=\hspace{0.2cm}&  \score_{u_i,u_j}(\ii,\para^{+k}) - \epsilon' f(\score_{u_i,u_j}(\ii,\para^{+k})) & \text{by Equation \eqref{eq:bikhod}}\nonumber \\
	&+ \epsilon'\min\{p,1-p\}(a - b)^2 \nonumber\\
	& -\epsilon'\frac{m-k-1}{m} \nonumber\\
	&-2(p-p')|a-b|\\
	=\hspace{0.2cm}&  \score_{u_i,u_j}(\ii,\para^{+k}) - \epsilon' \left[f(\score_{u_i,u_j}(\ii,\para^{+k})) + \frac{m-k}{m} \right]\nonumber \\
	&+ \epsilon'\left[\min\{p,1-p\}(a - b)^2 + 1/m \right]\nonumber\\
	&- 2(p-p')|a-b| \nonumber
	\end{align*}
	\endgroup
which is exactly the same as \eqref{eq:abargher} except for additional additive expressions of the last two lines. Thus, to complete the proof of Inequality \eqref{eq:abargher} we need to show 
\begin{equation}
\epsilon'\left[\min\{p,1-p\}(a-b)^2 + 1/m\right] \geq 2(p-p')|a-b|.\label{eq:kasif}
\end{equation}
Based on the values of $p$ and $a-b$ we consider the following three cases separately:

\begin{enumerate}
	\item $0 \leq |a-b| \leq m^{-1/2}$
	\item $0 \leq p \leq m^{-1/2}$
	\item $m^{-1/2} \leq |a-b| \leq 1 $ and $m^{-1/2} \leq p \leq 1$
\end{enumerate}

	\noindent \textbf{Case (i): $0 \leq |a-b| \leq m^{-1/2}$:} 
By Observation \ref{obs:tavan} and the fact that $p^{1+\frac{\epsilon m^{-1/2}}{4 \log m}} \leq p' \leq p$ we can imply $p-p' \leq \frac{\epsilon m^{-1/2}}{4 \log m}$. Therefore the right hand side of Inequality \eqref{eq:kasif} is bounded by
\begin{align*}
2(p-p')\big|a-b\big| &\leq 2\frac{\epsilon m^{-1/2}}{4 \log m}\big|a-b\big|\\
&= \frac{\epsilon m^{-1/2}}{2 \log m}\big|a-b\big|\\
&\leq \frac{\epsilon m^{-1/2}}{2\log m}m^{-1/2} & \text{since $|a-b| \leq m^{-1/2}$}\\
&= \frac{\epsilon}{2m\log m}\\
&= \frac{\epsilon'}{m \log m}\\
&\leq \frac{\epsilon'}{m} 
\end{align*}
which implies Inequality \eqref{eq:kasif} since the left hand side is at least $\epsilon'/m$.\\[0.5cm]

\noindent \textbf{Case (ii): $0 \leq p \leq m^{-1/2}$:} Let us first give a bound on the value of $p-p'$.
 \begin{align}
p-p' \leq& p-p^{1+\frac{\epsilon m^{-1/2}}{4 \log m}} \nonumber\\
\leq& \max\{\ln 1/p,1\}\min\{p,1-p\}\frac{\epsilon m^{-1/2}}{4 \log m}&\text{by Observation \ref{obs:tavan2}}\label{eq:balai}\\
\leq& (\ln \sqrt{m})m^{-1/2}\frac{\epsilon m^{-1/2}}{4 \log m}&\text{\eqref{eq:balai} is maximized for $p = m^{-1/2}$}\nonumber\\
\leq& m^{-1/2}\frac{\epsilon m^{-1/2}}{4}&\text{since $\log m \geq \ln \sqrt{m}$}\nonumber\\
=& \frac{\epsilon}{4m}&\nonumber\\
=& \frac{\epsilon'}{2m}&\text{since $\epsilon' = \epsilon/2$}\nonumber
\end{align}
 Also, $|a-b|$ is bounded by $1$ so the the right hand side is bounded by $\epsilon'/m$. Since the left hand side is at least $\epsilon'/m$ then Inequality \eqref{eq:kasif} holds.\\[0.5cm]
	
\noindent \textbf{Case (iii): $m^{-1/2} \leq |a-b| \leq 1$ and $m^{-1/2} \leq p \leq 1$:} In this case, we leverage Observation \ref{obs:tavan2} to show that 
 \begin{align*}
 p-p' \leq& p-p^{1+\frac{\epsilon m^{-1/2}}{4 \log m}} &\\ 
 \leq& \max\{\ln 1/p,1\}\min\{p,1-p\}\frac{\epsilon m^{-1/2}}{4 \log m}&\text{by Observation \ref{obs:tavan2}}\\
 \leq& (\ln \sqrt{m})\min\{p,1-p\}\frac{\epsilon m^{-1/2}}{4 \log m}&\text{since $p \geq m^{-1/2}$}\\
 \leq&\min\{p,1-p\}\frac{\epsilon m^{-1/2}}{4}&\text{since $\log m \geq \ln \sqrt{m}$}\\
 =&\min\{p,1-p\}\frac{\epsilon' m^{-1/2}}{2}&\text{since $\epsilon' = \epsilon/2$}\\
 \end{align*}
Therefore, the right hand side of Inequality \eqref{eq:kasif} can be bounded by 
\begin{align*}
2(p-p')|a - b| \leq & 2\min\{p,1-p\}\frac{\epsilon' m^{-1/2}}{2}|a-b| \\
= & \min\{p,1-p\}(\epsilon' m^{-1/2})|a-b| \\
\leq & \min\{p,1-p\}(\epsilon' m^{-1/2})|a-b|\frac{|a-b|}{m^{-1/2}} & \text{since $|a-b| \geq m^{-1/2}$}  \\
= & \epsilon' \min\{p,1-p\} |a-b|^2 \\
= & \epsilon' \min\{p,1-p\} (a-b)^2
\end{align*}
which proves Inequality \eqref{eq:kasif} since the left hand side is lower bounded by $\epsilon' \min\{p,1-p\} (a-b)^2$.\\[0.5cm]
	
So far, we have proven that Inequality \eqref{eq:abargher} holds for every pair of vertices $(u_i,u_j)$. All that remains is to show that Inequality \eqref{eq:abargher} implies Inequality \eqref{eq:gher}. To show this, we point out that by definition we have $$\score_{u_i}(\ii,\para'^{+k}) = \mathbb{E}_{u_j \sim N(u_i)}[\score_{u_i,u_j}(\ii,\para'^{+k})].$$
By applying Inequality \eqref{eq:abargher} we obtain:
\begingroup
\allowdisplaybreaks
\begin{align}
\score_{u_i}(\ii,\para'^{+k}) =& \mathbb{E}_{u_j \sim N(u_i)}\left[\score_{u_i,u_j}(\ii,\para'^{+k})\right]\nonumber\\
\geq & \mathbb{E}_{u_j \sim N(u_i)}\left[\score_{u_i,u_j}(\ii,\para^{+k}) - \epsilon'[\function(\score_{u_i,u_j}(\ii,\para^{+k}))+\frac{m-k}{m}] \right]\nonumber\\
= &  \mathbb{E}_{u_j \sim N(u_i)}\left[\score_{u_i,u_j}(\ii,\para^{+k})\right]\nonumber\\
& - \mathbb{E}_{u_j \sim N(u_i)}\left[\epsilon'[\function(\score_{u_i,u_j}(\ii,\para^{+k}))+\frac{m-k}{m}] \right]\nonumber\\
= &  \mathbb{E}_{u_j \sim N(u_i)}\left[\score_{u_i,u_j}(\ii,\para^{+k})\right]\nonumber\\
& - \epsilon' \mathbb{E}_{u_j \sim N(u_i)}\left[\function(\score_{u_i,u_j}(\ii,\para^{+k}))\right]\nonumber\\
& -\epsilon' \frac{m-k}{m}\nonumber\\
\geq &  \mathbb{E}_{u_j \sim N(u_i)}\left[\score_{u_i,u_j}(\ii,\para^{+k})\right]\label{eq:last}\\
& - \epsilon' \function\left(\mathbb{E}_{u_j \sim N(u_i)}[\score_{u_i,u_j}(\ii,\para^{+k})]\right)\nonumber\\
&-\epsilon' \frac{m-k}{m}&\nonumber\\
= &  \mathbb{E}_{u_j \sim N(u_i)}\left[\score_{u_i,u_j}(\ii,\para^{+k}) - \epsilon' \score_{u_i,u_j}(\ii,\para^{+k})\right]\nonumber\\
&-\epsilon' \frac{m-k}{m}&\nonumber\\
= &  \score_{u_i}(\ii,\para^{+k})) - \epsilon' \left[\function(\score_{u_i}(\ii,\para^{+k}))+ \frac{m-k}{m}\right]& \nonumber
\end{align}
\endgroup
which implies Inequality \eqref{eq:gher}. Inequality \eqref{eq:last} follows from Observation \ref{observation:ajib2}.
\end{proof}

Lemma \ref{lemma:challenge} suggests that we can have a discretized temperature set $T$ with size $O(\sqrt{m} \log m\log \emax)$ that can make  an almost optimal cooling schedule for any search graph. If we naively count the number of possible cooling schedules, then we obtain a bound of $(\sqrt{m} \log m\log \emax)^m$ which gives us the same upper bound as Corollary \ref{color:sample}. However, a better analysis can show that the number of possible cooling schedules limited to the temperatures in $T$ is bounded by
 $$(\sqrt{m} \log m \log \emax)^{\sqrt{m} \log m \log \emax} \binom{m}{\sqrt{m} \log m \log \emax}$$
 which gives us a sample complexity of $O_{\epsilon}(\sqrt{m} (\log m + \log \emax)).$

\begin{theorem}\label{theorem:main}
The sample complexity of computing an $\varepsilon$-approximately optimal cooling schedule with length $m$ is bounded by $O_{\epsilon}(\sqrt{m} (\log m + \log \emax)).$.
\end{theorem}

%% file: src/computational-model.tex
\newpage
\section{A Computational Model to Evaluate \textsf{SA} Algorithms}\label{sec:computational}
In this section, we introduce a model to evaluate the performance of an \textsf{SA} algorithm. The purpose of this model is to study the computational aspects of finding an optimal cooling schedule. We call this model \textit{the monotone stationary graph}. For simplicity, (and indeed without loss of generality as we show in Section \ref{sec:discretization}\footnote{A loss of $\epsilon > 0$ is incurred to the score of any algorithm in the discretized setting.}), we narrow down the space of the temperatures used in any algorithm to a finite set $T = \{d_{1},d_{2},d_{3},\ldots,d_{|T|}\}$. Therefore from here on, we focus our attention on the discretized temperatures in $T$ and assume that any algorithm (including any optimal solution) only uses temperatures in set $T$. Recall that every instance $\ii$ of the underlying problem translates to a search graph for our \textsf{SA} algorithm. The goal of the the monotone stationary graph is to represent the search graph in a compact manner so that we can evaluate the performance of a cooling schedule on each instance. Thus, monotone stationary graph is made by the search graph and may differ between different instances of the problem.

Recall that every state of an \textsf{SA} algorithm $\mathcal{A}$ corresponds to a distribution $\dtwo^{\mathcal{A}}$ over the vertices of the search graph. Initially, $\dtwo^{\mathcal{A}}$ is the same for all algorithms and shows the probability distribution over the vertices on which our algorithm initiates the search. One example is when our algorithm starts with a fixed node of the search graph in which case $\dtwo^{\mathcal{A}}$ is a deterministic distribution. Alternatively, $\dtwo^{\mathcal{A}}$ may be a uniform distribution when our algorithm starts with a random node of the search graph. As we perform more steps of the algorithm, $\dtwo^{\mathcal{A}}$ changes based on the criteria of the random walk and we hope that the correlation between $\dtwo^{\mathcal{A}}$ and the energy of the nodes becomes stronger. Ideally, we would like our algorithm to end up with a distribution $\dtwo^{\mathcal{A}}$ highly concentrated on the solution nodes.

Let us for every temperature $t \in |T|$, define a stationary distribution $\stationary_t$  which is a distribution of probabilities over the nodes of the search graph that an \textsf{SA} algorithm converges to after infinitely many steps of running on temperature $t$. Stationary distributions of simulated annealing are important and have been subject to a plethora of studies in the past decades ~\cite{eglese1990simulated,van1992job,mitra1986convergence,serafini1994simulated,henderson2003theory,aarts1987simulated}.
Intuitively, stationary distributions have positive correlation with the score of the nodes and as the temperature drops we expect the stationary distributions to provide higher (average) scores. Thus, the ideal case is when the state of our algorithm is very close to the stationary distribution for the lowest temperature for which the average score is the highest. The computational barrier is the \textit{convergence rate} of the distributions. An algorithm that starts from an initial distribution and runs on a temperature $t$ may need exponentially many steps to converge to the stationary distribution $\stationary_t$ whereas an algorithm that first reaches a stationary distribution $\stationary_{t'}$ for a higher temperature and then attempts to reach $\stationary_t$ may only need a small number of steps. This is perhaps best shown by the work of Wegener~\cite{wegener2005simulated} wherein the author showed that for the minimum spanning tree problem, a cooling schedule that gradually decreases the temperature is exponentially faster than a cooling schedule that repeats a certain temperature.   Thus, moving to intermediate stationary distributions may significantly improve the convergence rate of the algorithm.

Motivated by the above argument, we consider a model in which the states of any algorithm move between the stationary distributions. Let $d_1 > d_2 > d_3 > \ldots > d_{|T|}$ be all the distinct temperatures in $T$. We construct a graph with $|T|+1$ nodes $v_0,v_1, v_2, \ldots, v_{|T|}$ such that node $v_i$ corresponds to the set of all states close enough to the stationary distribution of temperature $d_i$. Also, $v_0$ is a special node corresponding to the initial distribution of the starting nodes. We assume that for every node $v_i$, the distances to the stationary distribution of temperature $d_i$ are so small such that the difference in the performance is negligible. Due to this assumption, our model features monotonicity. More precisely, a cooling schedule that repeats a temperature $t$ for 100 times is no better than the same cooling schedule that repeats $t$ for 101 times.

In our model, we add edges between the nodes to denote transitions between stationary distributions. The labels of these edges indicate the number of steps needed for transition between a node $v_i$ to a node $v_j$. 

\input{figs/node-transition}

Finally, we make one more assumption to complete the notion of monotonicity. If we have three temperatures $d_i > d_j > d_k$ the length of the edge from $v_i$ to $v_k$ is not smaller than the length of the edge from node $v_j$ to node $v_k$. Another interpretation of this property is the following: in order to reach the stationary distribution of a temperature  $d_k$, it is easier to start from the stationary distribution of a temperature closer to $d_k$ rather than a temperature with a much higher difference. Although for some very delicately constructed examples this may fail, the assumption is along the common perception for the behavior of the \textsf{SA} algorithms~\cite{aarts1987simulated,aarts1988simulated}. 

With the above definition, every path in the monotone stationary graph corresponds to a sequence of temperature which is made by the concatenation of the labels of the edges. A path can be traversed with a sequence of temperatures $\para$ if its corresponding label is a subsequence of $\para$. Given a sequence of temperatures $\para = \langle t_1, t_2, \ldots, t_m \rangle$, one can travel from node $v_0$ of the stationary distribution graph to a set of nodes via $\para$. In order to model the score of a cooling schedule $\para$, we assume that it takes us to the right most node $v_i$ such that there is a path from $v_0$ to $v_i$ whose label is a subsequence of $\para$. 
 Implicit to our model is the assumption that stationary distributions become better\footnote{More concentration on the solution nodes.} as the temperature drops. Thus, the scoring function gives us higher scores for lower temperatures.

For our computational results, we assume that the score of each cooling schedule is evaluated based on the above model. We compete against an optimal cooling schedule that uses a sequence of at most $m$ moves. Thus, we can assume w.l.o.g that the length of every (existing) edge is bounded by $m$. This along with the monotonicity property of our model implies that there is a trivial cooling schedule with $|T|m$ many moves that performs at least as well as the optimal schedule with $m$ steps. That is, in our model, a cooling schedule that contains $m$ copies of each temperature performs always as well as any cooling schedule of length $m$. Although we allow the size constraint to be violated by a small factor, our aim is to keep the length of our approximately optimal cooling schedule close to $m$.

Our model may raise a concern for a thoughtful reader. We only incorporate the types of algorithms whose states move between the stationary distributions. What if the optimal solution never gets close enough to some of the stationary distributions, yet moves towards them in order to reach the stationary distributions for lower temperatures (see Figure \ref{fig:optbad})?

\input{figs/optbad}

Although this may very well be the case in practice, the goal of this model is competing with the optimal algorithm that moves between the stationary distributions (and thus such a scenario is ruled out). We justify our model by the following intuitive argument: If moving towards a stationary distribution $\stationary_{d_j}$ makes a significant difference in the convergence rate for stationary distribution $\stationary_{d_k}$, it should be the case that a considerable portion of the path to the stationary distribution of $\stationary_{d_j}$ is already traversed. Thus, if we multiply the number of $d_j$ steps of the algorithm by a small constant, this algorithm should reach the stationary distribution of $\stationary_{d_j}$. In other words, the optimal algorithm that adheres to our model may not necessarily be the optimal algorithm, however, if we allow for more steps (by a multiplicative constant factor), we expect that the optimal algorithm of our model performs as well as the optimal algorithm in the unrestricted setting.

\subsection{Computational Results}
Although our model is general, we use the \textsf{SAT} problem to explain the terminology. Assume that the search graph contains $2^k$ vertices where every vertex is a true/false assignment to $k$ variables of the underlying problem. Every node of the search graph is associated with a value which we refer to as energy. This concept reflects how close this node is to a solution. One example of such energy function is the amount of clauses satisfied by that solution. Also, the score of a cooling schedule $\para$ is equal to the probability of finding a solution for the problem via simulated annealing using $\para$ as a cooling schedule. We model this quantity with the monotone stationary graph.

Recall that we are given a distribution $\mathcal{D}$ over a class of \textsf{SAT} instances and our aim is to design a learning algorithm that computes/approximates a cooling schedule with the highest average score. In other words, our goal is to find a cooling schedule $\para$ for simulated annealing that maximizes
$$\mathbb{E}_{\mathsf{I} \sim \distribution} [\score(\mathsf{I},\para)].$$

We model the performance of a simulated annealing algorithm by the monotone stationary graph explained previously. We compete against the score of the optimal cooling schedule with at most $m$ steps subject to our model. Notice that, the optimal cooling schedule may in fact get a higher score than what our model suggests but we only give credit to that schedule based on our model and not the actual likelihood of finding a solution. Nonetheless, our hope is that the difference between the practical results and our model is negligible.

We assume throughout this paper that the number of steps of the optimal cooling schedule is equal to $m$. However, in order to compete with the optimal solution, we allow more steps for our algorithm. We define an algorithm $\mathcal{A}$ to be $(\alpha,\epsilon)$-approximate, if the number of steps of $\mathcal{A}$ is bounded by $\alpha m$ and the average score of $\mathcal{A}$ differs from the optimal solution by at most an additive error of $\epsilon$. 

We show in Section \ref{sec:discretization} that from a sample-complexity standpoint, a learning algorithm only needs $\tilde O(\sqrt{m})$ samples from $\mathcal{D}$ (Theorem \ref{theorem:main}). This result is indeed not dependent on the monotone stationary graph. However, the computational complexity of the solution requires more assumption on the score of cooling schedules. To this end, we define four scoring systems and analyze each of the systems separately. 

Our results in Section \ref{sec:discretization} show that if we draw $n = \tilde O(\sqrt{m})$ samples from $\distribution$ and find the solution that maximizes the average score on these $n$ samples, the objective is approximately maximized for $\distribution$. Therefore, in the computational results, we assume that $n$ problem instances $\ii_1, \ii_2, \ldots, \ii_n$ are given and our goal is to find a sequence that maximizes the average score for those instances. We consider the following two settings for our problem:

\begin{itemize}
	\item \color{red}\textsf{separate paths}\color{black}: For each instance $\ii$, the optimal cooling schedule runs on a sequence of temperatures that move between the nodes of the stationary distribution graph. However, the sequence of stationary nodes may vary between different instances.
	
	\item \color{red}\textsf{identical paths}\color{black}: The optimal cooling schedule chooses a sequence $v_{a_1}, v_{a_2},\ldots, v_{a_x}$ of the nodes and does the following: starts and runs the algorithm by temperature $d_{a_1}$ so long as \textbf{all} instances reach stationary distribution $v_{a_1}$. Then, proceeds with applying temperature $d_{a_2}$ until all input instances reach stationary distribution $v_{a_2}$ and so on. In this case, the path taken in the stationary distribution graph is the same for all instances of the problem.
\end{itemize}

We bring an example to illustrate the difference of the two models. Consider a distribution $\distribution$ of the SAT instances which returns instances $\ione$ and $\itwo$ with equal probabilities. Let us assume that the monotone stationary graphs of the two instances are as shown in Figure \ref{fig:differ}. In the separate paths setting, the optimal sequence of temperatures that reaches the lowest stationary distribution for both instances is $\langle d_1, d_2, d_2, d_2, d_2\rangle$. Notice that in this case, for $\ione$ the path to $v_2$ is through $v_1$ but for $\itwo$ the path consists of a direct edge from $v_0$ to $v_2$. However, the choice of separate paths is not allowed in the identical paths model. Therefore, in the identical paths setting, the optimal solution is $\langle d_1,d_1,d_1,d_1, d_2, d_2, d_2, d_2\rangle$ which is through $v_1$ for both instances. 

\input{figs/differ}

Moreover, we also study a more restricted setting, in which in the optimal solution, all instances of the problem reach the stationary distribution for the lowest temperature. We call this setting the \color{red}\allsatisfied\color{black} setting.


For each combination of the settings we provide an algorithm along with its analysis. Table \ref{table:results} summarizes the time complexity of our algorithm in each setting.

\input{tables/computational-results}


One last thing to keep in mind before we go to the technical discussion is that monotone stationary graphs are not available to our algorithms. Therefore, the first step is to learn such a graph for a given instance $\ii$ of the problem. We begin by explaining this in Section \ref{sec:learning} and then bring our algorithms in Section \ref{sec:bestsequence}.

\input{figs/dag4}

To remind the reader of our assumptions, we bring a list of assumptions that we make for the model and the results:
\begin{itemize}
	\item (for the model): The state of any algorithm moves between stationary distributions.
	\item (for the model): For $i < j < k$ the length of the edge from $v_i$ to $v_k$ is not smaller than that of $v_j$ to $v_k$.
	\item (for the model): For any path $P$, a cooling schedule that contains the labels of the edges of the path as subsequence can take us to the end vertex. The score of a cooling schedule is equal to that of the best stationary distribution reachable via that schedule.
	\item (for the model): The score improves as $i$ increases for $v_i$.
	\item (in order to learn the monotone stationary graph): For every instance of the problem, there is a cooling schedule of length $m$ that takes us to node $v_{\tsize}$.
	\item (in order to learn the monotone stationary graph): There is a noticeable difference between the scores of the nodes of the monotone stationary graph. That is, by running $\poly(n,m,|T|)$ experiments we can tell whether two cooling schedules $\para_1$ and $\para_2$ take us to the same node or not.
\end{itemize}

%% file: figs/node-transition.tex
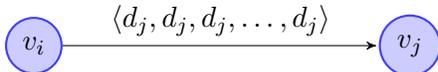
\begin{figure}[H]\centering
	\begin{tikzpicture}[node distance=1.3cm,>=stealth',bend angle=45,auto]
	
	\tikzstyle{place}=[circle,thick,draw=blue!75,fill=blue!20,minimum size=6mm]
	\tikzstyle{red place}=[place,draw=red!75,fill=red!20]
	\tikzstyle{transition}=[rectangle,thick,draw=black!75,
	fill=red,minimum size=4mm]
	
	\tikzstyle{every label}=[red]
	
	\begin{scope}
	  [align=center,node distance=5cm]  
	
	\node [place] (w1) {$v_i$};
	\node [place] (w2) [right of=w1] {$v_j$} edge [pre] node[swap] {$\langle d_j, d_j, d_j, \ldots, d_j\rangle$} (w1) ;
	
	\end{scope}
	
	\end{tikzpicture}
	\caption{A transition is shown between two graph nodes.}
\end{figure}

%% file: figs/optbad.tex
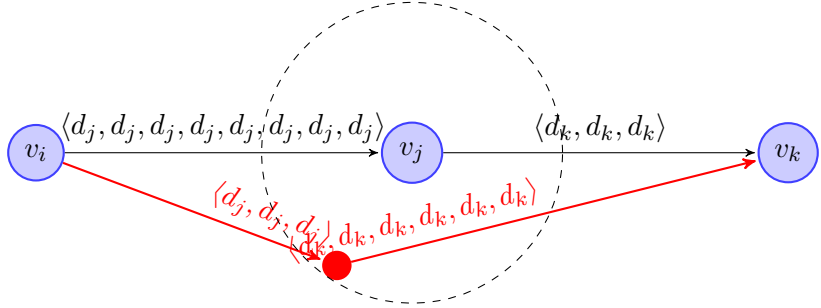
\begin{figure}[H]\centering
	\begin{tikzpicture}[node distance=1.3cm,>=stealth',bend angle=45,auto]
	
	\tikzstyle{place}=[circle,thick,draw=blue!75,fill=blue!20,minimum size=6mm]
	\tikzstyle{rd}=[circle,thick,draw=red,fill=red,minimum size=1mm]
	\tikzstyle{red place}=[place,draw=red!75,fill=red!20]
	\tikzstyle{transition}=[rectangle,thick,draw=black!75,
	fill=red,minimum size=4mm]
	
	\tikzstyle{every label}=[red]
	
	\begin{scope}
	  [align=center,node distance=5cm]  
	
	\node [place] (w1) {$v_i$};
	\node [place] (w2) [right of=w1] {$v_j$} edge [pre] node[swap] {$\langle d_j, d_j, d_j, d_j, d_j, d_j,d_j,d_j\rangle$} (w1) ;
	\node [rd] (w4) at (4,-1.5) {} edge [pre,color=red,thick] node[swap,color=red,rotate=-21] {$\langle d_j, d_j, d_j\rangle$} (w1);
	\node [place] (w3) [right of=w2] {$v_k$} edge [pre] node[swap] {$\langle d_k, d_k, d_k\rangle$} (w2) edge [pre,color=red,thick] node[swap,color=red,rotate=14] {$\langle  d_k, d_k, d_k, d_k, d_k, d_k\rangle$} (w4) ;
	\end{scope}

	\draw [dashed] (w2) circle (2cm);

	\end{tikzpicture}
	\caption{Red edges show the cooling schedule of the optimal solution. In this case, the optimal solution moves toward the stationary distribution of temperature $d_j$ without getting close enough to its stationary distribution. }\label{fig:optbad}
\end{figure}

%% file: figs/differ.tex
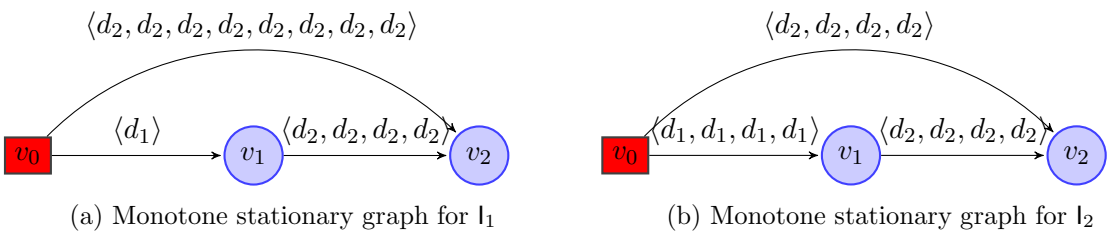
\begin{figure}[H]
		\centering
	\begin{subfigure}[t]{0.45\textwidth}
	\begin{tikzpicture}[node distance=3cm,>=stealth',bend angle=45,auto]

\tikzstyle{place}=[circle,thick,draw=blue!75,fill=blue!20,minimum size=6mm]
\tikzstyle{red place}=[place,draw=red!75,fill=red!20]
\tikzstyle{transition}=[rectangle,thick,draw=black!75,
fill=red,minimum size=4mm]

\tikzstyle{every label}=[red]

[align=center,node distance=3cm]  

\node [transition] (w1) {$v_0$};
\node [place] (w2) [right of=w1] {$v_1$} edge [pre] node[swap] {$\langle d_1\rangle$} (w1) ;
\node [place] (w3) [right of=w2] {$v_2$} edge [pre] node[swap] {$\langle d_2, d_2,d_2,d_2\rangle$} (w2) edge [pre,bend right] node[swap] {$\langle d_2, d_2,d_2,d_2,d_2,d_2,d_2,d_2\rangle$} (w1) ;



[node distance=6cm]  


\end{tikzpicture}\caption{Monotone stationary graph for $\ii_1$}
	\end{subfigure}
	\quad
	\begin{subfigure}[t]{0.45\textwidth}
	\begin{tikzpicture}[node distance=3cm,>=stealth',bend angle=45,auto]

\tikzstyle{place}=[circle,thick,draw=blue!75,fill=blue!20,minimum size=6mm]
\tikzstyle{red place}=[place,draw=red!75,fill=red!20]
\tikzstyle{transition}=[rectangle,thick,draw=black!75,
fill=red,minimum size=4mm]

\tikzstyle{every label}=[red]

[align=center,node distance=6cm]  


\node [transition] (w'1) [below of=w1] {$v_0$};
\node [place] (w'2) [right of=w'1] {$v_1$} edge [pre] node[swap] {$\langle d_1,d_1,d_1,d_1\rangle$} (w'1) ;
\node [place] (w'3) [right of=w'2] {$v_2$} edge [pre] node[swap] {$\langle d_2, d_2,d_2,d_2\rangle$} (w'2) edge [pre,bend right] node[swap] {$\langle d_2, d_2,d_2,d_2\rangle$} (w'1) ;


[node distance=3cm]  


\end{tikzpicture}\caption{Monotone stationary graph for $\ii_2$}

	\end{subfigure}
	\caption{An example to show the difference between the separate paths setting and identical paths setting.}\label{fig:differ}
\end{figure}

%% file: tables/computational-results.tex
\begin{table}[!htbp]
	\centering
	\begin{tabular}{|c|c|c|}
		\hline
		  \textsf{identical paths} & \textsf{separate paths}  &\textsf{separate paths} + \allsatisfied\\
		\hline
		 exact solution & exact solution &$(O(\log n|T|),0)$ approximation\\
  in time  & in time & in time \\
$\mathsf{poly}(m,n,|T|)$ & $\mathsf{poly}(m,n,|T|^n)$ &$\mathsf{poly}(m,n,|T|)$\\
		\hline
	\end{tabular}
\end{table}\label{table:results}

%% file: figs/dag4.tex
\begin{figure}[h!]\centering
	\begin{tikzpicture}[node distance=1.3cm,>=stealth',bend angle=45,auto]
	
	\tikzstyle{place}=[circle,thick,draw=blue!75,fill=blue!20,minimum size=6mm]
	\tikzstyle{red place}=[place,draw=red!75,fill=red!20]
	\tikzstyle{transition}=[rectangle,thick,draw=black!75,
	fill=red,minimum size=4mm]
	
	\tikzstyle{every label}=[red]
	
	\begin{scope}
	  [align=center,node distance=3cm]  
	
	\node [transition] (w1) {$v_0$};
	\node [place] (w2) [right of=w1] {$v_1$} edge [pre] node[swap] {$\langle d_1, d_1, d_1\rangle$} (w1);
	\node [place] (w3) [right of=w2] {$v_2$} edge [pre] node[swap] {$\langle d_2, d_2\rangle$} (w2) edge [pre,bend right] node[swap] {$\langle d_2, d_2, d_2, d_2\rangle$} (w1);
	\node [place] (w4) [right of=w3] {$v_3$}  edge [pre, bend left] node[swap] {$\langle d_3, d_3, d_3, d_3\rangle$} (w2) edge [pre] node[swap] {$\langle d_3, d_3\rangle$} (w3);
	\node [place] (w5) [right of=w4] {$v_4$} edge [pre] node[swap] {$\langle d_4, d_4, d_4\rangle$} (w4);
	\node [place] (w6) [right of=w5] {$v_5$} edge [pre] node[swap] {$\langle d_5, d_5\rangle$} (w5) {} edge [pre, bend right] node[swap] {$\langle d_5, d_5, d_5,d_5, d_5, d_5\rangle$} (w2);
	\end{scope}

	\end{tikzpicture}
	\caption{An example of the monotone stationary graph is shown. Only non-trivial edges are shown in this figure. For instance, an edge of length $6$ from $v_3$ to $v_5$ can be implied from the edge $(v_1,v_5)$.}\label{fig:daggood}
\end{figure}
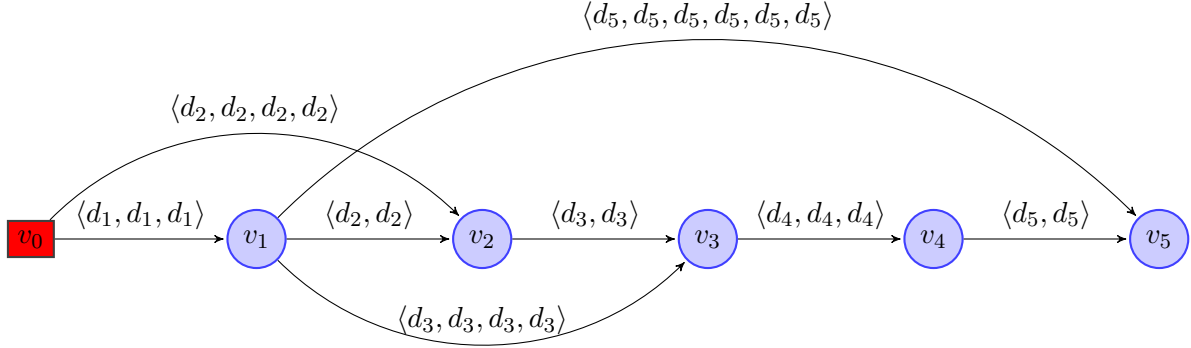

%% file: src/learning.tex
\section{Learning A Monotone Stationary Graph}\label{sec:learning}
In this section, we show how one can learn the monotone stationary graph for a particular instance of the problem in polynomial time. Recall that $\score(\para,\mathsf{I})$ denotes the success probability of finding a solution to the problem. We make two assumptions to learn the monotone stationary graph. The first assumption is that for every instance there exists a sequence of length $m$ that takes us to the optimal node (corresponding to the lowest temperature) in the monotone stationary graph. The second assumption is that there is a noticeable difference between the score of the nodes. That is, for any two cooling schedules $\para_1$ and $\para_2$ we can tell whether they take us to the same node in the monotone stationary graph or they take us to different nodes by running the \textsf{SA} algorithm several times and comparing their success ratios. 

The discretization of the the temperatures is w.l.o.g as we show in Section \ref{sec:discretization}. Also, we ignore all edges whose sizes are more than $m$. This obviously does not hurt the optimal solution since its length is bounded by $m$.

\begin{observation}\label{obs:1}
	Given a sequence $\para$ of temperatures, we can verify in polynomial time whether $\para$ takes us to $v_{\tsize}$. That is, we can answer in polynomial time whether $\para$ is at least as good as any other sequence or not.
\end{observation}
\begin{proof}
   By the above assumptions, a sequence that contains $m$ repetitions of each temperature has to take us to node $v_{\tsize}$ (otherwise there is no path of length $m$ to $v_{\tsize}$). Thus, by running the algorithm on this sequence, we can learn the average score of that node in time $\poly(m,|T|)$. Now, for a sequence $\para$, we just need to run the algorithm several times and verify whether the success rate is close to $s_{\tsize}$ or not. 
\end{proof}

Using Observation \ref{obs:1}, we can construct the monotone stationary graph for an instance $\mathsf{I}$ of the problem.
\begin{lemma}
	Given an instance $\mathsf{I}$ of the problem, one can construct the underlying graph $\graph(\mathsf{I})$ in time $\poly(m,|T|)$.
\end{lemma}
\begin{proof}
	We assume that there is a path of length $m$ from node $v_0$ to node $v_{\tsize}$. Thus, a sequence containing $m$ repetitions of each temperature takes us there. Now, imagine we wish to answer the following question:
	
	``\textit{Is there an edge from node $v_i$ to node $v_{\tsize}$ with label $\overbrace{\langle d_{\tsize}, d_{\tsize}, \ldots, d_{\tsize}\rangle}^{k}$?}''\\[0.5cm]
		
	To answer the above question, we can construct a sequence of temperatures that contains $m$ repetitions of all temperatures $d_1,d_2,\ldots,d_i$. Next, we add $k$ repetitions of temperature $d_{\tsize}$ to the end of this sequence. If this sequence takes us to node $v_{\tsize}$, then there is an edge from $v_i$ to $v_{\tsize}$ with a label that contains at most $k$ copies of $d_{\tsize}$. Thus, we can answer the query with a binary search.
	
	Using the above machinery, we can extract all the edges that end at node $v_{\tsize}$. Based on this information, we can find the smallest path that takes us from node $v_{\tsize-1}$ to node $v_{\tsize}$ and then recursively solve the problem for node $v_{\tsize-1}$. With the same argument, we can discover all of the edges for all vertices of the graph.
\end{proof}

%% file: src/bestsequence.tex
\section{Computing/Approximating the Optimal Cooling Schedule}\label{sec:bestsequence}
The problem that we are concerned with in this section is computing (or approximating) the optimal cooling schedule for a set of problem instances. More precisely, let $\ii_1, \ii_2,\ldots, \ii_n$ be $n$ instances of the problem whose monotone stationary graphs are available. The goal here is to find a cooling schedule whose size is close to $m$ and and whose average score for the $n$ instances is close to the optimal solution.

\subsection{Identical Paths} 
The easier setting that we study is identical paths. In this setting, we compete with the optimal solution that chooses the same path for all instances. In other words, in such solutions, we fix a set of nodes $v_{a_1}, v_{a_2},\ldots, v_{a_k}$ and find a cooling schedule that takes all instances through this path. More precisely, we put enough temperatures $d_{a_1}$ to make sure all instance reach vertex $v_{a_1}$. Next, we proceed by doing the same thing for $v_{a_2}$ and so on.

We show that in this setting the problem of finding the optimal cooling schedule reduces to shortest path. Construct a graph $G$ with the same vertex set as the monotone stationary graphs. We put a directed edge from vertex $v_i$ to vertex $v_j$ of $G$, if and only if such an edge exists in the corresponding monotone stationary graphs of all instance. Moreover, we set the length of this edge as the largest length in all the graphs. Finally, we find the lowest temperature $d_i$ (meaning that $i$ is maximized) such that vertex $v_i$ is reachable from vertex $v_0$ via a path of length at most $m$. We prove that this algorithm is optimal.

\begin{algorithm}\caption{Exact algorithm for the identical paths setting.}
\begin{algorithmic}[1]
		\FOR {$ 1\leq k \leq n $} 
		 \STATE Let $r_{i,j}^{(k)}$ be the minimum number of repetitions for temperature $d_j$ to make a transition from node $v_i$ to $v_j$ on the $k$-th monotone stationary graph.
		 \ENDFOR
		 \STATE $r_{i,j} := \max_{k \in [n]} r_{i,j}^{(k)} $. 
	\STATE Let $r_{i,j} = 0 $ when $i\geq j$. 
	\STATE Construct a graph with vertices $v_0, \cdots, v_{\tsize}$ and pairwise distances $r_{i,j}$.
	\STATE Find the the shortest paths from $v_0$ to all nodes of the graph.
	\STATE Find the largest $q$ such that node $v_q$ is within a distance of $m$ from $v_0$.
	\STATE Output the optimal sequence of temperatures made by this path.
\end{algorithmic}
\end{algorithm}

\begin{theorem}\label{theorem:guli}
	Given $n$ monotone stationary graphs for $n$ instances of the problem, one can find in polynomial time a cooling schedule of length $m$ that maximizes the average score for the $n$ instances in the same paths setting
\end{theorem}
\begin{proof}
The proof is based on the fact that the path is the same for all instances. Thus, in order to make a transition from vertex $v_i$ to vertex $v_j$, one needs to add as many copies of temperature $d_j$ as the size of the largest label among all instances. Thus, the furthest we can get from node $v_0$ is a node $v_i$ whose distance from $v_0$ is bounded by $m$ on the graph we make.
\end{proof}

\subsection{Separate Paths}

The more challenging setting is when we allow the instances to have different paths in the optimal solution. In this case, the problem is much harder since we have to consider $n$ different monotone stationary graphs and solve the problem with respect to all of them. However, it is not hard to see that if $n$ is constant, one can find the optimal cooling schedule of length $m$ in polynomial time.

\begin{lemma}\label{lemma:constant}
	Given $n$ monotone stationary graphs for $n$ instances of the problem, one can find in time $\tilde O(m|T|^{n+1})$ a cooling schedule of length $m$ that maximizes the average score for the $n$ instances in the separate paths setting
\end{lemma}
\begin{proof}
	The proof is similar to the proof of Theorem \ref{theorem:guli}. However, since we may traverse a different path for every graph, we need to construct our graph more carefully. To this end, our vertex set would be the multiplication of the vertex sets for the monotone stationary graphs. That is, we put $(|T|+1)^n$ different vertices in our graph such that every vertex shows one combination of the nodes for the instances.
	
	Every vertex has $O(m|T|)$ different edges that shows how the combination changes by adding $1 \leq i \leq m$ copies of temperature $d_j$ to the sequence. Finally, we compute the distances of all vertices from node $(v_0,v_0,\ldots,v_0)$ and find the one whose distance is bounded by $m$ and its score is maximized. Then, we recover the path to that node and report it.
\end{proof}

Obviously, the runtime of Lemma \ref{lemma:constant} is not polynomial when $n$ is super constant. Therefore, for asymptotically larger $n$'s, we present a polynomial time algorithm that approximates the solution. Our algorithm works for the \textsf{all-satisfied} setting, which means that there is an optimal solution that brings all instances to the vertex corresponding to the lowest temperature. Our algorithm loses a polylogarithmic factor in the size of the sequence but obtain the same score as the optimal solution with high probability.

Let us assume for simplicity that the score for each instance $\ii_k$ is equal to $1$ if and only if our sequence takes us to node $v_{\tsize}$ in its monotone stationary graph. Otherwise the score is equal to 0. Our algorithm is \textbf{not} dependent on this assumption, yet it makes the explanation much simpler. We begin by an observation that translates the definition of score into the set cover setting.

We say a cooling schedule $\para$ is an \textit{acceptable cooling schedule} for an instance $\ii$ of the problem if $\para$ takes us all the way to node $v_{\tsize}$ in its monotone stationary graph. Define an edge from a vertex $v_i$ to a vertex $v_j$ to be \textit{crossing} for a vertex $v_k$ if $i < k \leq j$ holds. Moreover, we say a sequence $\para$ \textit{encompasses} an edge from $v_i$ to $v_j$ from a particular monotone stationary graph if $\para$ contains at least as many repetitions of temperature $d_j$ as the label of the edge from $v_i$ to $v_j$. 
An example of the definitions is shown in Figure \ref{fig:example}.

\input{figs/example}

Now, we are ready to state an observation that plays an important role in our algorithm.
\begin{observation}\label{obs:important}
	A sequence of temperatures $\para$ is acceptable for an instance $\ii$ of the problem if and only if for every $1 \leq i \leq \tsize$, $\para$ encompasses at least one crossing edge with $v_i$.
\end{observation}

\begin{proof}
The necessity of the condition is trivial. If $\para$ does not encompass a crossing edge for a vertex $v_i$, then $\para$ cannot reach vertex $v_i$ in monotone stationary graph. The vice versa also holds. Suppose for the sake of contradiction that a sequence $\para$ encompasses a crossing edge for every vertex but it does not take us to node $v_{\tsize}$. In this case, there exists a vertex $v_i$, such that all vertices $v_{i-1}$ is reachable but none of the vertices $v_j$ is reachable for $j \geq i$ are reachable using $\para$. This means that $\para$ does not encompass an edge that crosses vertex $v_i$ otherwise we could have reached vertex $v_i$ using $\para$.
\end{proof}

We are now ready to state the main theorem of this section.

\begin{theorem}\label{theorem:log}
	Let $\ii_1$, $\ii_2, \ldots, \ii_n$ be $n$ monotone stationary graph with the guarantee that there exists a cooling schedule of length $m$ that is acceptable for all instance. One can find in polynomial time a cooling schedule for the \textsf{SA} algorithm whose average score is equal to that of the optimal cooling schedule of size $m$. Our algorithm is randomized and gives a solution with probability at least $1-e^{-100}$. Also, the average size of the cooling schedule is bounded by $O(m (\log \tsize + \log n))$.
\end{theorem}
\input{src/bestsecproof}

%% file: figs/example.tex
\begin{figure}\centering
	\begin{tikzpicture}[node distance=1.3cm,>=stealth',bend angle=45,auto]
	
	\tikzstyle{place}=[circle,thick,draw=blue!75,fill=blue!20,minimum size=6mm]
	\tikzstyle{red place}=[place,draw=red!75,fill=red!20]
	\tikzstyle{transition}=[rectangle,thick,draw=black!75,
	fill=red,minimum size=4mm]
	
	\tikzstyle{every label}=[red]
	
	\begin{scope}
	  [align=center,node distance=3cm]  
	
	\node [transition] (w1) {$v_0$};
	\node [place] (w2) [right of=w1] {$v_1$} edge [pre] node[swap] {$\langle d_1, d_1, d_1\rangle$} (w1);
	\node [place] (w3) [right of=w2] {$v_2$} edge [pre,dashed,color=red] node[swap] {$\langle d_2, d_2\rangle$} (w2) edge [pre,bend right,dashed,color=red] node[swap] {$\langle d_2, d_2, d_2, d_2\rangle$} (w1);
	\node [place] (w4) [right of=w3] {$v_3$}  edge [pre, bend left,dashed,color=red] node[swap] {$\langle d_3, d_3, d_3, d_3\rangle$} (w2) edge [pre,color=red] node[swap] {$\langle d_3, d_3\rangle$} (w3);
	\node [place] (w5) [right of=w4] {$v_4$} edge [pre,color=red] node[swap] {$\langle d_4, d_4, d_4\rangle$} (w4);
	\node [place] (w6) [right of=w5] {$v_5$} edge [pre,color=red] node[swap] {$\langle d_5, d_5\rangle$} (w5) {} edge [pre, bend right,dashed] node[swap] {$\langle d_5, d_5, d_5,d_5, d_5, d_5\rangle$} (w2);
	\end{scope}

	\end{tikzpicture}
	\caption{All the dashed edges are crossing for vertex $v_2$. Moreover, sequence $\langle \color{blue} d_1, d_1,\color{red}d_2,d_2,d_2,d_2,\color{blue},d_3,d_3,d_3,d_3,\color{red},d_4,d_4,d_4,d_4,\color{blue}d_5,d_5\color{black}\rangle$ is an acceptable sequence that encompasses all the red edges. The edges that are not drawn are implied by the edges depicted in the figure. }\label{fig:example}
\end{figure}
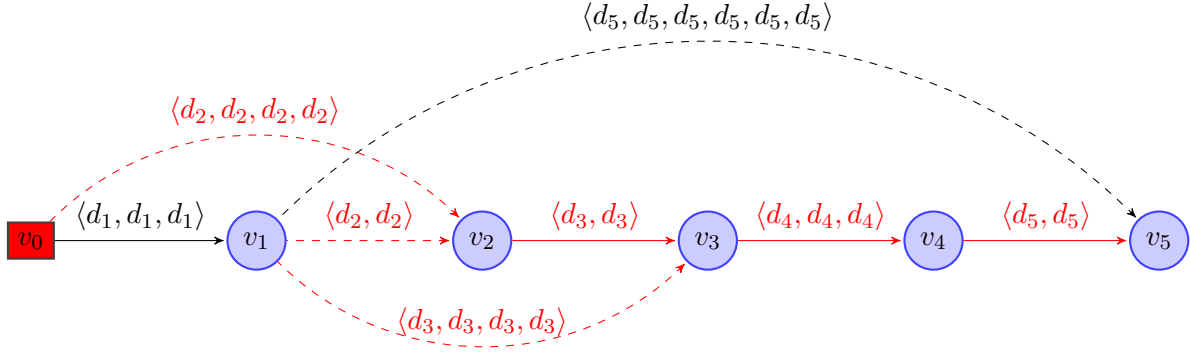

%% file: src/bestsecproof.tex
\begin{proof}
Observation \ref{obs:important} gives us a strong tool to analyze the solution. Let $\opt$ be the optimal cooling schedule of size $m$ which is acceptable for all instances. Due to Observation \ref{obs:important}, $\opt$ encompasses at least one crossing edge for all vertices of all monotone stationary graphs. To formalize this, define a set 
\begin{equation*}
\begin{split}
S =& \{\langle d_1 \rangle, \langle d_1, d_1 \rangle, \ldots,  \overbrace{\langle d_1, d_1, \ldots, d_1\rangle}^{m}\} \cup\\
& \{\langle d_2 \rangle, \langle d_2, d_2 \rangle, \ldots,  \overbrace{\langle d_2, d_2, \ldots, d_2\rangle}^{m}\} \cup\\
& \hspace{3cm}\vdots\\
& \{\langle d_{\tsize} \rangle, \langle d_{\tsize}, d_{\tsize} \rangle, \ldots,  \overbrace{\langle d_{\tsize}, d_{\tsize}, \ldots, d_{\tsize}\rangle}^{m}\}
\end{split}
\end{equation*}
to be the set of all possible repetitions for all temperatures and for each element $e \in S$, define $\length(e)$ to be the size of $e$. In addition to this, for each element $e \in S$, define $c_e = \{0,1\}$ to be equal to $1$ if and only if $\opt$ contains $\length(e)$ repetitions of the character corresponding to $e$. 

\input{figs/example2}

To clarify the definitions, consider an example with only a single instance shown in Figure \ref{fig:example2}. In this case, $m = 9$ and the optimal sequence of temperatures is $\opt = \langle \color{blue}d_1, d_1, d_1, \color{red}d_5, d_5, d_5, d_5, d_5, d_5\color{black}\rangle$. In this case $S$ contains $45 = m|T|$ elements out of which only $c_{\langle d_1, d_1, d_1\rangle}$ and $c_{\langle d_5, d_5, d_5, d_5, d_5, d_5\rangle}$ are equal to $1$. Moreover, $\length(\langle d_1, d_1, d_1\rangle) = 3$ and $\length(\langle d_5, d_5, d_5, d_5, d_5, d_5\rangle) = 6$ hold.

For a vertex $v_i$ in graph $\ii_k$, define $\cross(\ii_k, v_i)$ to be the set of elements in $S$ that correspond to the crossing edges of $v_i$. This way, the optimal solution of the problem can be formulated via the following integer feasibility program:

\begin{equation}\label{eq:ip}
\begin{array}{ll@{}ll}
	\text{constraints: }& \sum \length(e) c_e &\leq m   & \\
							& \sum_{e \in \cross(v_i,\ii_k)} c_e\hspace{0.5cm}\color{white}.\color{black} &\geq 1   & \forall 1 \leq k \leq n\text{ and }1 \leq i \leq \tsize\\
							& c_e &\in \{0,1\}& \forall e \in S\\
\end{array}
\end{equation}
where the variables of the program are $c_e$'s. Indeed by relaxing the conditions of IP \ref{eq:ip} we can obtain LP \ref{eq:lp}.

\begin{equation}\label{eq:lp}
\begin{array}{ll@{}ll}
\text{constraints: }& \sum \length(e) c_e &\leq m   & \\
& \sum_{e \in \cross(v_i,\ii_k)} c_e\hspace{0.5cm}\color{white}.\color{black} &\geq 1   & \forall 1 \leq k \leq n\text{ and }1 \leq i \leq \tsize\\
& 0 \leq c_e \leq 1& & \forall e \in S\\
\end{array}
\end{equation}

Now we solve LP \ref{eq:lp} and construct a solution as follows: for each element $e \in S$, we add $e$ to our solution independently with probability $ \min\{\alpha c_e, 1\} $, where $\alpha = 100 (\log \tsize + \log n)$. 

First, it's easy to see the expected length of our solution is bounded by $\alpha m$:

\[\mathbb{E} [\text{length}] 
= \sum_{e \in S} \Pr[e \text { is picked}] \ell(e)
= \sum_{e \in S} \min(\alpha c_e, 1) \ell(e) 
\leq \sum_{e \in S} \alpha c_e \ell(e)
\leq \alpha m\]

where the last step is due to the constraint in LP. 

Next, we will show that with high probability, the resulting sequence is acceptable for each instance $\ii_k$. Consider the case when the resulting sequence is not acceptable for $\ii_k$. By Observation \ref{obs:important}, there exists a $v_i$ such that none of the edges in $\cross(v_i,\ii_k)$ were encompassed in our solution.
By union bound, the probability of this bad event can be upper bounded by

\begin{equation} \label{eq:nod_satisfied}
	\Pr [\ii_k \text{ is not satisfied}] 
	\leq \sum_{i=1}^{\tsize} \Pr [\forall e \in \cross(v_i,\ii_k) \text{ wasn't picked}]
\end{equation}
Now we focus on the probability inside the summation. Since each element was selected independently, this probability equals to

\begin{align*}
	\Pr [\forall e \in \cross(v_i,\ii_k) \text{ wasn't picked}] =& \prod_{e \in \cross(v_i,\ii_k)} \Pr [ e\text{ wasn't picked}]\\ 
	=& \prod_{e \in \cross(v_i,\ii_k)} \max(1-\alpha c_e,0).
\end{align*}
If $\alpha c_e \geq 1$ for some $e\in \cross(v_i,\ii_k)$, then the probability is $0$. Otherwise, since $1-x \leq e^{-x}$, we have
\[
\prod_{e \in \cross(v_i,\ii_k)} \max(1-\alpha c_e,0) \leq e^{- \sum_{e \in \cross(v_i,\ii_k)} \alpha c_e} \leq e^{-\alpha}
\]
Therefore, by \ref{eq:nod_satisfied}
\[\Pr [\ii_k \text{ is not satisfied}] 
\leq \tsize  e^{-\alpha} \]
Using union bound again, we have
\begin{align*} \Pr [\text{any instance }\ii_k\text{  is not satisfied}] \leq &\sum_{1 \leq k \leq n} \Pr [\ii_k \text{ is not satisfied}] \\
\leq &\tsize n e^{-\alpha} \\
\leq &e^{-100}
\end{align*}

 Hence, we proved that with probability $1-e^{-100}$,  the resulting sequence is acceptable for each instance $\ii_k$.
\end{proof}

%% file: figs/example2.tex
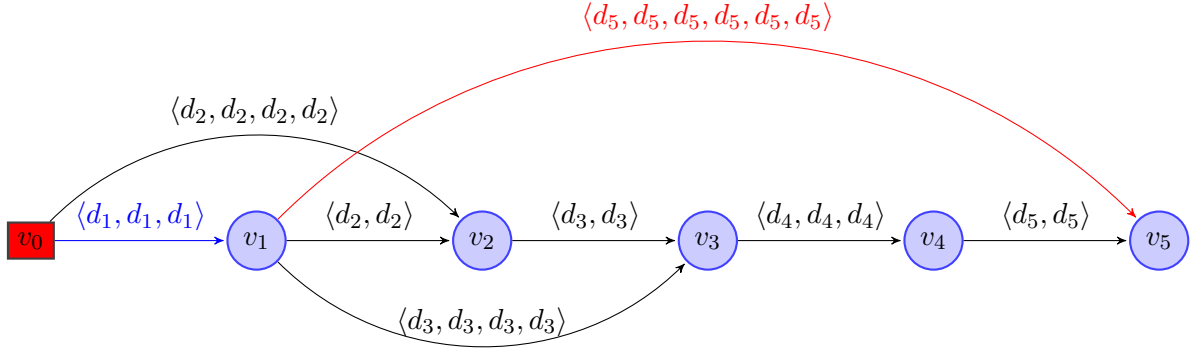
\begin{figure}\centering
	\begin{tikzpicture}[node distance=1.3cm,>=stealth',bend angle=45,auto]
	
	\tikzstyle{place}=[circle,thick,draw=blue!75,fill=blue!20,minimum size=6mm]
	\tikzstyle{red place}=[place,draw=red!75,fill=red!20]
	\tikzstyle{transition}=[rectangle,thick,draw=black!75,
	fill=red,minimum size=4mm]
	
	\tikzstyle{every label}=[red]
	
	\begin{scope}
	  [align=center,node distance=3cm]  
	
	\node [transition] (w1) {$v_0$};
	\node [place] (w2) [right of=w1] {$v_1$} edge [pre,color=blue] node[swap] {$\langle d_1, d_1, d_1\rangle$} (w1);
	\node [place] (w3) [right of=w2] {$v_2$} edge [pre] node[swap] {$\langle d_2, d_2\rangle$} (w2) edge [pre,bend right] node[swap] {$\langle d_2, d_2, d_2, d_2\rangle$} (w1);
	\node [place] (w4) [right of=w3] {$v_3$}  edge [pre, bend left] node[swap] {$\langle d_3, d_3, d_3, d_3\rangle$} (w2) edge [pre] node[swap] {$\langle d_3, d_3\rangle$} (w3);
	\node [place] (w5) [right of=w4] {$v_4$} edge [pre] node[swap] {$\langle d_4, d_4, d_4\rangle$} (w4);
	\node [place] (w6) [right of=w5] {$v_5$} edge [pre] node[swap] {$\langle d_5, d_5\rangle$} (w5) {} edge [pre, bend right,color=red] node[swap] {$\langle d_5, d_5, d_5,d_5, d_5, d_5\rangle$} (w2);
	\end{scope}

	\end{tikzpicture}
	\caption{For $m = 9$ the optimal sequence of temperatures is $\langle \color{blue}d_1, d_1, d_1, \color{red}d_5, d_5, d_5, d_5, d_5, d_5\color{black}\rangle$. All the edges skipped in the figured can be implied from the edges shown by monotonicity.}\label{fig:example2}
\end{figure}

%% file: src/lowerbound-appendix.tex
\newpage
\section{Omitted Proofs of Section \ref{sec:lower}}
\begin{proof}[of Observation \ref{obs:main}]
	For  Observation \ref{obs:main}.(i), define $\tau_r=\{\arg \min_{i\geq 0}: x_i = r\}$. We would like to prove that, 
	\begin{equation}
		\Pr[\tau_{\sqrt{k/c}} \leq k] \geq 0.95
    \end{equation} 
	By the Markov property, $\tau_r$ is the sum of $r$ independent copies of $\tau_1$. Let the probability generating function of $\tau_r$ be $F_r(z):=\E[z^{\tau_r}] = \sum_{j=0}^{\infty} \Pr(\tau_r = j) z^j$, then we have $F_r(z) = F_1(z)^r$. Furthermore, we have the following recurrence about $F_1(z):$
	
	\begin{equation}
		F_1(z) = \frac{z}{3} \left( 1 + F_1(z) + F_1(z)^2 \right)
	\end{equation}
	
	Hence, 
	\begin{equation}
		F_1(z) = \frac{(3-z) - \sqrt{3(z+3)(1-z)}}{2z}
	\end{equation}
	
	One important property about $F_1(z)$ is that for all $z \in [0,1]$, 
	\begin{equation}
		F_1(z) \leq 1 - \sqrt{1-z}
	\end{equation}
	By Markov's inequality,
	\begin{align}
		\Pr[\tau_r \geq k] &\leq \inf_z \frac{\E[z^{\tau_r} ] }{z^k}\\
		&= \inf_z \frac{(1-\sqrt{1-z})^r}{z^k}\\
		(z:= 1- \frac{1}{k}) &= (1-\frac{1}{\sqrt{k}})^r (1-\frac{1}{k})^{-k} \\
		(r:=c\sqrt{k})&\leq \exp(-c-1)
	\end{align}
	Hence we have completed the proof. 
	
	For  Observation \ref{obs:main}.(ii) and (iii), we need a classical result in martingale concentration inequalities, the Freedman's inequality for scalar martingales~\cite[Thm.~(1.6)]{freedman1975tail}, see also \cite[Thm.~(1.1)]{tropp2011freedman}.
	\begin{theorem}[Freedman] \label{thm:freedman}
		Consider a real-valued martingale $\{ Y_k : k = 0, 1, 2, \dots \}$ with difference sequence $\{ X_k : k = 1, 2, 3, \dots \}$.  Assume that the difference sequence is uniformly bounded:
		$$
		|X_k| \leq R
		\quad\text{almost surely}
		\quad\text{for $k = 1, 2, 3, \dots$.}
		$$
		Define the predictable quadratic variation process of the martingale:
		$$
		W_k := \sum\nolimits_{j=1}^k \E_{j-1} \big(X_j^2\big)
		\quad\text{for $k = 1, 2, 3, \dots$.}
		$$
		Then, for all $t \geq 0$ and $\sigma^2 > 0$,
		$$
		\Pr[\exists k \geq 0 : Y_k \geq t \ \text{ and }\ 
			W_k \leq \sigma^2 ]
		\leq \exp \left\{ - \frac{ t^2/2 }{\sigma^2 + Rt/3} \right\}.
		$$
	\end{theorem}
	
	When the difference sequence $\{ X_k \}$ consists of independent random variables, the predictable quadratic variation is no longer random.  In this case, Freedman's inequality reduces to the usual Bernstein inequality.
	
	To prove \ref{obs:main}, we let $y_k = x_k + k\gamma$ where $\gamma = p_b - p_f$, then $y_k$ is a martingale since the difference sequence $\Delta_k = y_k - y_{k-1}$ has expectation zero. Furthermore, the difference sequence $\Delta_k$ is uniformly bounded with $|\Delta_k| \leq R:= 1 + |\gamma|$, and $W_k = \sum_{j=1}^{k} \E_{j-1}(\Delta_k^2) \leq k R^2$. By the Freedman's inequality, 
	
	\begin{equation}\label{eq:freedman}
\Pr[\max_{k} y_k \geq t]
\leq \exp \left\{ - \frac{ t^2/2 }{kR^2 + Rt/3} \right\}.
	\end{equation}
		
	 Since we have $R = 1+|\gamma|$ and $y_k = x_k + k \gamma$, \eqref{eq:freedman} is equivalent to 
	\begin{equation}
\Pr[\max_{k} x_k \geq t- k\gamma] \leq \exp \left\{ - \frac{ t^2/2 }{kR^2 + tR/3} \right\}.
	\end{equation}
	Let $t = 3R \sqrt{k\log k}  $, we have 
	$$ \frac{ t^2/2 }{kR^2 + Rt/3} = \frac{ t^2/(2R^2) }{k+t/(3R)} \geq \frac{ t^2/(2R^2) }{k+k}  = \frac{t^2}{4kR^2}  \geq 2 \log k$$
	Rearranging terms gives 
	 \begin{equation}
	 \Pr[\max_{k} x_k \geq 3(1+|\gamma|)\sqrt{k\log k} - k\gamma] \leq \frac{1}{k^2}.
	 \end{equation}
	 When $p_f = p_b = p_s = \frac{1}{3}$, we have $\gamma = 0$ and the above inequality is equivalent to 
	 \begin{equation}
	 \Pr[\max_{k} x_k \geq 3\sqrt{k\log k}] \leq \frac{1}{k^2}.
	 \end{equation}
	 This proves (ii). For (iii), we note that $\gamma = p_b - p_f  \geq\frac{2c \log k'}{\sqrt{k'}}$ and we only needs to prove that 
	 \begin{equation}
	 3(1+|\gamma|)\sqrt{k\log k} - k\gamma \leq \frac{\sqrt{k'}}{2}
	 \end{equation}
	 Note that 
	 \begin{equation}
	 3(1+|\gamma|)\sqrt{k\log k} - k\gamma \leq 3\sqrt{k\log k} - \frac{k}{2} \gamma.
	 \end{equation}
	 When $\gamma \geq\frac{6\sqrt{\log k}}{\sqrt{k}}$, RHS is negative so the inequality holds trivially. Otherwise, we have $\gamma = 2c\frac{\log k'}{\sqrt{k'}} < \frac{6\sqrt{\log k}}{\sqrt{k}}$, hence $k' \geq O(1) $ $\frac{k}{\log^2 k} > \sqrt{k}$. By AM-GM inequality, 
	 \begin{equation}
		3\sqrt{k\log k}  \leq \frac{k}{2} \gamma + \frac{9 \log k}{\gamma}
	 \end{equation}
	 Therefore, we have 
	 \begin{equation}
	 	3(1+|\gamma|)\sqrt{k\log k} - k\gamma \leq \frac{9 \log k}{\gamma} = \frac{9 \log k}{2c \log k'} \sqrt{k'} \leq \frac{9}{c} \sqrt{k'}
	 \end{equation}
	 Hence, let $c=9$ completes the proof.
\end{proof}

%% file: src/improved-upperbound-appendix.tex
\newpage
\section{Omitted Proofs of Section \ref{sec:improved-upper}}

\begin{proof}[of Observation \ref{observation:ajib1}]
    The proof is given below:
    \begin{align}
	    p\function(x)+(1-p)\function(y) = & p \left[ x - x^2 \right] + (1-p) \left[y - y^2\right] \nonumber \\
	    = & (px + (1-p)y) - (px + (1-p)y)^2 - \left[(p-p^2) (x^2 + y^2 - 2xy) \right] \nonumber \\
	    = & (px + (1-p)y) - (px + (1-p)y)^2 - \left[(p-p^2) (x-y)^2 \right] \nonumber \\
	    \leq & (px + (1-p)y) - (px + (1-p)y)^2 - \left[\min\{p,1-p\} (x-y)^2 \right]\label{equation:kazai} \\
	    = & \function(px + (1-p)y) - \left[\min\{p,1-p\} (x-y)^2 \right]. \nonumber
    \end{align}
Inequality \eqref{equation:kazai} follows from the fact that both $p$ and $1-p$ are in range $[0,1]$ and thus $p(1-p) \leq \min\{p,1-p\}$.
\end{proof}

\begin{proof}[of Observation \ref{obs:sadeh}]
	\begin{align*}
	|f(x) - f(y)| = & |[x - x^2] - [y - y^2]|\\
	= & |(x-y) - (x^2 - y^2)|\\
	= & |(x-y) - (x-y)(x+y)|\\
	= & |(x-y)(1-(x+y))|\\
	= & |x-y| |(1-(x+y))|\\
	\leq & |x-y|
	\end{align*}
where the last inequality holds since $0 \leq x+y \leq 2$ and therefore $-1 \leq 1 - (x+y) \leq 1$ which implies $0 \leq |1 - (x+y)| \leq 1$.
\end{proof}

\begin{proof}[of Observation \ref{obs:tavan}]
	To prove the observation, we take the first derivative of $p - p^{1+x}$ which is equal to
	\begin{align*}
	\frac{d}{dp}\left[p - p^{1+x}\right] = 1-(1+x)p^{x}
	\end{align*}
	which means that the function is maximized (or minimized) at $p_0 = (1+x)^{-1/x}$. It is easy to see that since $p - p^{1+x}$ is non-negative in range $[0,1]$ and is equal to $0$ at both $p=0$ and $p=1$ then the expression should be maximized at $p_0$. 
	Thus, the maximum value for $p - p^{1+x}$ is bounded by
	\begin{align*}
	p_0 - p_0^{1+x} &= p_0 (1-p_0^{x} ) \\
	&\leq 1-p_0^{x} \\
	&= 1-((1+x)^{-1/x})^{x}\\
	&= 1-(1+x)^{-1}\\
	&= 1-1/(1+x)\\
	&= x/(1+x)\\
	&\leq x.\\
	\end{align*}
\end{proof}

\begin{proof}[of Observation \ref{obs:tavan2}]
	We first show the proof for the case of $p \leq 1/2$. We start by the famous inequality $1 + y \leq e^y$~\cite{mitzenmacher2017probability} which holds for any $y \in \mathbb{R}$. Therefore, we have $1-e^y \leq -y$. By setting $y = -x \ln 1/p$ we obtain 
	$$1-e^{-x \ln 1/p} \leq x \ln 1/p$$
	Notice that $e^{-x \ln 1/p}$ can be written as $(e^{-\ln 1/p})^x = (e^{\ln p})^x = p^x$. Thus, we have
	$$1-p^x \leq x \ln 1/p$$
	Multiplying both sides by $p$ gives us 
	$$p-p^{1+x} \leq p (\ln 1/p) x$$ which proves the observation for $p \leq 1/2$.
	Next, we show the statement for $p \geq 1/2$. In this case, we prove $p - p^{1+x} \leq (1-p)x$ which implies the observation. Our goal here is to prove $p - p^{1+x} - (1-p)x \leq 0$ for any $p \in [0.5,1]$ and any $0 \leq x \leq 1$. Thus, we take the derivative of $p$ to bound its maximum value.
	$$\frac{d}{dp} \left[ p - p^{1+x} - (1-p) x \right] = 1  - (1+x)p^x + x$$
	which is equal to $0$ only at $p_0 = 1$. At $p_0$ we have $p-p^{1+x} - (1-p)x = 0$ which is not greater than $0$. Also, since for $p = 0$ the expression $p-p^{1+x} - (1-p)x$ is equal to $-x$ which is negative, it means that the function is maximized at $p = 1$. Thus, $p-p^{1+x} - (1-p)x$ is always upper bounded by $0$ which means $p-p^{1+x} \leq (1-p)x$.
\end{proof}